\documentclass[10pt]{article}

\usepackage[utf8]{inputenc}
\usepackage[T1]{fontenc}

\usepackage{epsf}
\usepackage{amsmath}

\allowdisplaybreaks

\usepackage[showframe=false]{geometry}
\usepackage{changepage}

\usepackage{epsfig}
\usepackage{amssymb}

\usepackage{amsthm}
\usepackage{setspace}
\usepackage{cite}
\usepackage{mcite}

\usepackage{algorithmic}  
\usepackage{algorithm}

\usepackage{shadow}
\usepackage{fancybox}
\usepackage{fancyhdr}

\usepackage{color}
\usepackage[usenames,dvipsnames,svgnames,table]{xcolor}
\newcommand{\bl}[1]{\textcolor{blue}{#1}}
\newcommand{\red}[1]{\textcolor{red}{#1}}

\definecolor{mypurple}{rgb}{.4,.0,.5}
\newcommand{\prp}[1]{\textcolor{mypurple}{#1}}

\usepackage[hyphens]{url}

\usepackage[colorlinks=true,
            linkcolor=black,
            urlcolor=blue,
            citecolor=purple]{hyperref}

\usepackage{breakurl}

\def\w{{\bf w}}

\def\y{{\bf y}}

\def\x{{\bf x}}

\def\x{{\mathbf x}}

\def\w{{\bf w}}

\def\x{{\bf x}}
\def\y{{\bf y}}
\def\z{{\bf z}}
\def\q{{\bf q}}

\def\b{{\bf b}}

\def\d{{\bf d}}
\def\f{{\bf f}}

\def\tr{\mbox{Tr}}

\def\tr{{\rm tr}\,}

\def\cS{{\mathcal S}}

\def\be{\begin{equation}}
\def\ee{\end{equation}}
\def\ba{\left[\begin{array}}
\def\ea{\end{array}\right]}

\def\w{{\bf w}}

\def\x{{\bf x}}
\def\y{{\bf y}}
\def\z{{\bf z}}
\def\q{{\bf q}}

\def\b{{\bf b}}

\def\d{{\bf d}}
\def\f{{\bf f}}

\def\1{{\bf 1}}

\def\g{{\bf g}}
\def\0{{\bf 0}}

\def\erf{\mbox{erf}}
\def\erfc{\mbox{erfc}}

\def\mR{{\mathbb R}}

\def\mE{{\mathbb E}}

\def\mP{{\mathbb P}}

\def\lp{\left (}
\def\rp{\right )}

\def\w{{\bf w}}

\def\y{{\bf y}}

\def\x{{\bf x}}

\def\x{{\mathbf x}}

\def\w{{\bf w}}

\def\x{{\bf x}}
\def\y{{\bf y}}
\def\z{{\bf z}}
\def\q{{\bf q}}

\def\b{{\bf b}}

\def\d{{\bf d}}
\def\f{{\bf f}}

\def\tr{\mbox{Tr}}

\def\tr{{\rm tr}\,}

\def\be{\begin{equation}}
\def\ee{\end{equation}}
\def\ba{\left[\begin{array}}
\def\ea{\end{array}\right]}

\def\w{{\bf w}}

\def\x{{\bf x}}
\def\y{{\bf y}}
\def\z{{\bf z}}
\def\q{{\bf q}}

\def\b{{\bf b}}

\def\d{{\bf d}}
\def\f{{\bf f}}

\def\({\left (}
\def\){\right )}

\def\1{{\bf 1}}

\def\q{{\bf q}}

\def\g{{\bf g}}
\def\0{{\bf 0}}

\usepackage{xcolor}
\usepackage{color}

\definecolor{darkgreen}{rgb}{0, 0.4,0}
\newcommand{\dgr}[1]{\textcolor{darkgreen}{#1}}

\definecolor{purplebrown}{rgb}{0.5,0.1,0.6}

\definecolor{ultclupcol}{rgb}{0.1,0.5,0.5}

\definecolor{mytrycolor}{rgb}{0.5,0.7,0.2}

\definecolor{ultclupcola}{rgb}{.5,0,.5}

\definecolor{shadebrown}{rgb}{0.1,0.1,0.9}
\definecolor{lightblue}{rgb}{0.2,0,1}

\usepackage{fancybox}
\usepackage{graphicx}
\usepackage{epstopdf}
\usepackage{epsfig}
\usepackage{wrapfig}
\usepackage{subfigure}

\usepackage{xcolor}
\usepackage{tcolorbox}
\tcbuselibrary{skins}

\newtcbox{\xmybox}{on line,
arc=7pt,
before upper={\rule[-3pt]{0pt}{10pt}},boxrule=0pt,
boxsep=0pt,left=6pt,right=6pt,top=0pt,bottom=0pt,enhanced, coltext=blue, colback=white!10!yellow}

\newtcbox{\xmyboxa}{on line,
arc=7pt,
before upper={\rule[-3pt]{0pt}{10pt}},boxrule=0pt,
boxsep=0pt,left=6pt,right=6pt,top=0pt,bottom=0pt,enhanced, colback=white!10!yellow}

\newtcbox{\xmyboxb}{on line,
arc=7pt,
before upper={\rule[-3pt]{0pt}{10pt}},boxrule=1pt,colframe=darkgreen!100!blue,
boxsep=0pt,left=6pt,right=6pt,top=0pt,bottom=0pt,enhanced, colback=white!10!yellow}

\newtcbox{\xmyboxc}{on line,
arc=7pt,
before upper={\rule[-3pt]{0pt}{10pt}},boxrule=.7pt,colframe=blue!100!blue,
boxsep=0pt,left=6pt,right=6pt,top=0pt,bottom=0pt,enhanced, coltext=blue, colback=white!10!yellow}

\newtcbox{\xmytboxa}{on line,
arc=7pt,
before upper={\rule[-3pt]{0pt}{10pt}},boxrule=.0pt,colframe=pink!50!yellow,
boxsep=0pt,left=6pt,right=6pt,top=0pt,bottom=0pt,enhanced, coltext=white, colback=blue!40!red}

\newtcbox{\xmytboxb}{on line,
arc=7pt,
before upper={\rule[-3pt]{0pt}{10pt}},boxrule=.0pt,colframe=pink!50!yellow,
boxsep=0pt,left=6pt,right=6pt,top=0pt,bottom=0pt,enhanced, coltext=white, colback=white!40!green}

\makeatletter
\newcommand\subsubsubsection{\@startsection{paragraph}{4}{\z@}{-2.5ex\@plus -1ex \@minus -.25ex}{1.25ex \@plus .25ex}{\normalfont\normalsize\bfseries}}
\newcommand\subsubsubsubsection{\@startsection{subparagraph}{5}{\z@}{-2.5ex\@plus -1ex \@minus -.25ex}{1.25ex \@plus .25ex}{\normalfont\normalsize\bfseries}}
\makeatother

\newtheorem{theorem}{Theorem}

\newtheorem{lemma}{Lemma}

\begin{document}

\begin{singlespace}

\title {Fixed width treelike neural networks capacity analysis -- generic activations 
}
\author{
\textsc{Mihailo Stojnic
\footnote{e-mail: {\tt flatoyer@gmail.com}} }}
\date{}
\maketitle

\centerline{{\bf Abstract}} \vspace*{0.1in}

We consider the capacity of \emph{treelike committee machines} (TCM) neural networks. Relying on Random Duality Theory (RDT), \cite{Stojnictcmspnncaprdt23} recently introduced a generic framework for their capacity analysis. An upgrade based on the so-called \emph{partially lifted} RDT (pl RDT) was then presented in \cite{Stojnictcmspnncapliftedrdt23}. Both lines of work focused on the networks with the most typical, \emph{sign}, activations. Here, on the other hand, we focus on networks with other, more general, types of activations and show that the frameworks of \cite{Stojnictcmspnncaprdt23,Stojnictcmspnncapliftedrdt23} are sufficiently powerful to enable handling of such scenarios as well. In addition to the standard \emph{linear} activations, we uncover that particularly convenient results can be obtained for two very commonly used activations, namely, the \emph{quadratic} and \emph{rectified linear unit (ReLU)} ones. In more concrete terms, for each of these activations, we obtain both the RDT and pl RDT based memory capacities upper bound characterization for \emph{any} given (even) number of the hidden layer neurons, $d$. In the process, we also uncover the following two, rather remarkable, facts: 1) contrary to the common wisdom, both sets of results show that the bounding capacity decreases for large $d$ (the width of the hidden layer) while converging to a constant value; and 2) the maximum bounding capacity is achieved for the networks with precisely \textbf{\emph{two}} hidden layer neurons! Moreover, the large $d$ converging values are observed to be in excellent agrement with the statistical physics replica theory based predictions.

\vspace*{0.25in} \noindent {\bf Index Terms: TCM neural networks; Capacity; Lifted random duality theory; Different activations}.

\end{singlespace}

\section{Introduction}
\label{sec:intro}

Demand for efficient handling and interpretation of large data sets, has grown rather rapidly over the last 15-20 years. Machine learning (ML) clearly distinguished itself as a particularly helpful set of concepts capable of providing the needed technical and scientific resources to meet such a high demand. Consequently, a fast development of various ML branches ensued. Neural networks (NN), as one of such branches, quickly became one of the focuses of strong research interests and among the fastest growing research fields. Many great results and quite a few remarkable breakthroughs that relate to both theoretical and practical NN aspects have been obtained. Moreover, many of the well known results achieved in prior decades -- that for a long time served as academic prototypes -- have been revisited and brought to practical usability. Among the most prominent of them are certainly those that relate to one of the key NN features, the so-called, network's \emph{memory capacity} (see, e.g., \cite{Schlafli,Cover65,Winder,Winder61,Wendel62,Cameron60,Joseph60,Gar88,Ven86,BalVen87}). Here we continue the same trend and focus on several important capacity related questions and provide a strong theoretical progress. Before discussing in more detail some of the main problems and our technical contributions, we find it convenient to first introduce the basics of the NN models of our interest.

\subsection{Feed forward neural networks --  mathematical basics}
\label{sec:model}

To be able to properly introduce the network's memory capacity as the main object of our interest, we first discuss the underlying network architecture.

\noindent \textbf{\emph{Architecture:}} We are interested in multilayered multi-input single-output feed-forward neural networks with $L-2$ hidden layers and $d_i$ ($i\in\{1,2,\dots,L\}$) nodes (neurons) in the $i$-th layer. To ensure a notational facilitation, two additional layers, indexed by  $i=1$ and $i=L$ are artificially added and they correspond to the network input and output, respectively  (although the input and output of the network are basically artificial NN layers, to be in agreement with the introduced indexation, we will refer to them as networks layers). The way the network operates is basically determined by specifying
 the vectors of threshold functions, $\f^{(i)}(\cdot)=[\f_1^{(i)}(\cdot),\f_2^{(i)}(\cdot),\dots,\f_{d_{i+1}}^{(i)}(\cdot)]^T$. Each threshold function $\f_{j}^{(i)}(\cdot):\mR^{d_i}\rightarrow \mR$ describes how neuron $j$  in layer $i$ operates. The network effectively functions by taking the outputs of the nodes from layer $i$ as the inputs of the nodes in layer $i+1$ and transforming them into the new outputs (in layer $i+1$) via a linear combination governed by the matrix of weights $W^{(i)}\in\mR^{d_{i}\times d_{i+1}}$. After setting $\d=[d_1,d_2,\dots,d_L]$ (with $d_1=n$ and $d_L=1$) and denoting by $\b^{(i)}\in\mR^{d_{i+1}},i=1,2,\dots L$ the so-called thresholds vectors and by $\x^{(i)}\in\mR^{d_i}$ and $\x^{(i+1)}\in\mR^{d_{i+1}}$ the inputs and outputs of the neurons in layer $i$, one
 has the following:
\begin{center}
     	\tcbset{beamer,nobeforeafter,lower separated=false, fonttitle=\bfseries, coltext=black,
		interior style={top color=yellow!20!white, bottom color=yellow!60!white},title style={left color=black, right color=red!50!blue!60!white},
		before=,after=\hfill,fonttitle=\bfseries,equal height group=AT}
\begin{tcolorbox}[title=Mathematical formalism of NN with architecture $A(\d\text{, }\f^{(i)})$:]
\vspace{-.03in}$\mbox{\textbf{input:}} \triangleq \x^{(1)}\quad \longrightarrow$ \hfill
\tcbox[coltext=black,colback=white!65!red!30!orange,interior style={top color=yellow!20!white, bottom color=yellow!60!white},nobeforeafter,box align=base]{$ \x^{(i+1)}=\f^{(i)}(W^{(i)}\x^{(i)}-\b^{(i)}) $ }\hfill $\longrightarrow \quad \mbox{\textbf{output:}} \triangleq \x^{(L+1)}$.
\vspace{-.2in}\begin{equation}\label{eq:model0}
\vspace{-.2in}\end{equation}
\vspace{-.4in}\end{tcolorbox}
\end{center}
The architecture of the network, $A(\d;\f^{(i)})$,  is fully specified by the vectors $\d$ and $\f^{(i)}$. Also, when the vectors of functions, $\f^{(i)}$, are identical, we write $A(\d;\f)$ instead of $A(\d;\f^{(i)})$.

\noindent \textbf{\emph{Memory capacity:}} As mentioned earlier, one of the most fundamental features of any neural net (including single neurons as special cases) is their ability to properly memorize/store a large amount of data. A simple way to see how the above formalism achieves this is the following:  assume the existence of $m$ data pairs $(\x^{(0,k)},\y^{(0,k)})$, $k\in\{1,2,\dots,m\}$, with $\x^{(0,k)}\in \mR^{n}$ being the $n$-dimensional data vectors and $\y^{(0,k)}\in\mR$ being their associated labels. Determining matrices $W^{(i)}$  such that  
\begin{equation}\label{eq:model3}
\x^{(1)}=\x^{(0,k)}\quad \Longrightarrow \quad \x^{(L+1)}=\y^{(0,k)} \qquad \forall k,
\end{equation}
is then sufficient to properly relate given data vectors to their corresponding labels. If the network architecture $A(\d,\f^{(i)})$ is given then its \emph{memory capacity, $C(A(\d,\f^{(i)}))$,} is  defined as the largest sample size, $m$, such that (\ref{eq:model3}) holds for any collection of data pairs $(\x^{(0,k)},\y^{(0,k)})$, $k\in\{1,2,\dots,m\}$ with certain prescribed properties. Since the memory capacity plays one of the most important roles in understanding the overall neural nets' functioning mosaic, finding both, its precise theoretical characterization and the corresponding fast algorithmic procedure that achieves it, is of utmost importance. Of our particular interest in this paper are the theoretical aspects and we below provide a host of results that for many well known architectures almost fully characterize their capacities.

To facilitate the presentation, a few structural and technical assumptions are in place as well. Many of them, however, are aligned with the ones discussed in \cite{Stojnictcmspnncaprdt23,Stojnictcmspnncapliftedrdt23}. To avoid an unnecessary repetition, we only briefly recall on these and refer for a more detailed exposition to \cite{Stojnictcmspnncaprdt23,Stojnictcmspnncapliftedrdt23}. On the other hand, we place the most emphasis on those that are substantially different and particularly relevant to the results that we present in this paper.

\subsection{Technical assumptions}
\label{sec:assumpt}

To facilitate the exposition and to make the final results cleaner and easier to use, we will rely on several architectural and data related assumptions. We state them below before starting the analytical considerations. As they are fairly common and rather prevalent in the literature, we avoid discussing them in deep details.

\noindent \textbf{\emph{Network architecture assumptions:}} 1-hidden layer \emph{treelike committee machine} type of neural networks with \textbf{\emph{generic}} zero-threshold activation functions in the hidden layer are considered. This  means the following: \emph{\textbf{(i)}} We assume $L=3$, $\b^{(i)}=0$, $W^{(1)}=I_{n\times n}$, and $W^{(3)}=\w^T$, where $\w\in\mR^{d_2\times 1}$ (i.e. $W^{(3)}$ is a $d_2$-dimensional row vector that will be specified depending on the type of the considered activation functions). \emph{\textbf{(ii)}}  We also define $d\triangleq d_2$ and $\delta\triangleq \delta_1=\frac{d_1}{d_2}=\frac{n}{d}$. While the presented mathematical concepts will hold for any $d$, to be able to get concrete capacity values, we eventually assume that $d$ is any (even) natural number.
\textbf{\emph{(iii)}} In the first layer, we consider \emph{identity} neuronal functions, i.e. we take $\f^{(1)}(\x^{(1)})=\x^{(1)}$. In the hidden layer, we take the \emph{generic} zero-threshold activations $\f_j^{(2)}(\cdot): \mR^{d_2}\rightarrow R$ with $\f_j^{(2)}(\cdot)=\f_k^{(2)}(\cdot)$ for any $j\neq k$. Finally, at the output, we take zero-threshold sign activation $\f^{(3)}(W^{(3)}\x^{(3)}-\b^{(3)})=\mbox{sign} \left ( W^{(3)}\x^{(3)}\right )$. For notational simplicity, we denote this architecture by $A(\d;[\f^{(1)},\f^{(2)},\f^{(3)}])=A(\d;[I,\f^{(2)},\mbox{sign}])
\triangleq A(\d;\f^{(2)})$. \textbf{\emph{(iv)}}
 The matrix $W^{(i)}$ can be generically full or with a particular structure. Both types of structuring have been of interest throughout the literature. A particular type of sparse structuring, where the support of $W^{(i)}$'s $j$-th row, $\mbox{supp}\lp W_{j,:}^{(i)}\rp$, satisfies
$\mbox{supp}\lp W_{j,:}^{(i)}\rp=\cS^{(j)}$, with $\cS^{(j)}\triangleq\{(j-1)\delta+1,(j-1)\delta+2,\dots,j\delta\}$, makes the above architecture correspond to what is in the literature referred to as the \emph{treelike committee machines} (TCM). Precisely such architectures will be of our interest in this paper. For the completeness, we also add that if the matrix is full then the above architecture corresponds to what is in the literature typically referred to  as the \emph{fully connected committee machines} (FCM).

\noindent \textbf{\emph{Data related assumptions:}} \emph{\textbf{(i)}} \emph{Binary} labeling $\y_i^{(0,k)}\in\{-1,1\}$, as the most standard type of labeling, is assumed as well (choosing \emph{sign} perceptron as neuronal function at the output naturally complements the binary labeling choice as well). \emph{\textbf{(ii)}} Inseparable data sets are not allowed (for example, indistinguishable/contradictory pairs (or subgroups) like $(\x^{(0,k)},\y^{(0,k)})$ and $(\x^{(0,k)},-\y^{(0,k)})$ can not appear). \emph{\textbf{(iii)}} Data sets of statistical nature will be of our main interest. We particularly focus on $\x^{(0,k)}$ as iid standard normals. This follows into the footsteps of the trend established in the classical single perceptron references (see, e.g., \cite{DTbern,Gar88,StojnicGardGen13,Cover65,Winder,Winder61,Wendel62}) and allows for, expectedly, a fairly universal statistical treatment. It is also useful to note that for providing universal capacity upper bounds, any type of acceptable data set (including even nonstatistical ones) actually suffices.


\subsection{Prior work}
\label{sec:prior}

The problems of our interest are well known and have been studied in various forms for almost 70 years. Naturally, the early studies related to the single neurons while the more recent ones emphasize the importance of understanding the multi-neuron or multi-layered architectures. Given the pace at which the ML and NN fields are developing, the underlying relevant literature is rather vast and growing. We below highlight the results that we view as most closely related to our own.

Since the memory capacity of spherical perceptrons is directly connected to several fundamental questions in integral geometry, the early capacity considerations were related to some of the geometrical/probabilistic  classic works (see, e.g., \cite{Schlafli,Cover65,Wendel62,Joseph60}). Possibly the most famous of them states that the capacity of the spherical \emph{sign} perceptrons \emph{doubles the dimension} of the data ambient space, $n$, i.e., $C(A(1;\mbox{sign}))\rightarrow 2n$ as $n\rightarrow\infty$. After being initially obtained as a remarkable closed form combinatorial geometry fact in \cite{Schlafli,Cover65,Winder,Winder61,Wendel62,Cameron60,Joseph60}, it was decades later rediscovered and reproved in various different forms in a plethora of different fields ranging from machine learning and pattern recognition to information theory, probability, and statistical physics  (see, e.g., \cite{BalVen87,Ven86,DT,StojnicISIT2010binary,DonTan09Univ,DTbern,Gar88,StojnicGardGen13,StojnicGardSphErr13}).

\underline{\emph{Sign perceptrons networks:}} Despite the elegance of the single perceptron results, the corresponding multi-perceptron ones are not easy to obtain. Particularly scarce are the TCM related ones. While a bit more is known about the FCM ones, a direct connection between the two is not apparent. Besides the trivial fact that the FCM capacities upper-bound the corresponding TCM ones, one may also (somewhat ad-hoc) view the TCM capacities as roughly the FCM ones divided by $d$. Although non necessarily rigorous (or even fully correct) such a viewing suggests a potential usefulness of FCM results. Still, an overwhelming majority of known results indicates that the memory capacity is unavoidably related to the total number of the network weights, $w=\sum_{i=1}^{L-1} d_id_{i+1}$. For example, the famous VC-dimension \emph{qualitative} memory capacity upper bound gives the scaling $O(w\log(w))$. It is interesting to note that for NNs with 1-hidden layer, $w=d_1d_{2}+d_2=(n+1)d$ for FCM and $w=d_1+d_{2}=n+d$ for TCM. This, on the other hand, for large $d_i$'s and huge $n$, gives, the above mentioned, ``\emph{division by $d$}'' FCM -- TCM capacity relation. A couple of lower bounding results are known as well. For example, \cite{Baum88} argued that the capacity of a shallow 3-layer network (similar to the one studied here) scales as $O(nd)$. On the other hand, \cite{Vershynin20} obtained recently a stronger version, where, for the networks with more than three layers, the capacity is shown to be roughly at least $O(w)$.

Obtaining precise results of \textbf{\emph{non-scaling}} type in network architectures turned out to be a much harder challenge. This seems particularly surprising given the simplicity and elegance of the corresponding single perceptron ones. Until the very recent appearance of \cite{Stojnictcmspnncapliftedrdt23,Stojnictcmspnncaprdt23} there was hardly any mathematically rigorous result that could provide even remotely close sign perceptron networks capacity estimates.  \cite{Stojnictcmspnncaprdt23} utilized the Random duality theory (RDT) and developed a generic framework for the analysis of TCM network capacities. As a results of the framework, strong upper bounds were obtained for any given (odd) number of the nodes in the hidden layer. \cite{Stojnictcmspnncapliftedrdt23} went a step further, utilized a partially lifted RDT variant (pl RDT) and substantially lowered the upper bounds of \cite{Stojnictcmspnncaprdt23}.

\underline{\emph{Different activations networks:}} Given the importance of a single sign perceptron, studying their merging into a large architectural structure is a natural transition. Two things should be kept in mind when making such a transition though. First, it is not clear a priori that the results that hold for the single perceptron will hold in a similar fashion for the network of perceptrons. Second, the sign perceptrons are not continuous functions and handling them algorithmically when using or training the network might impose computationally/numerically unsurpassable obstacles. If on top of that, one adds the above mentioned analytical hardness, the need for potentially less simple but easier to use activation factions is rather obvious.

As the discreteness is typically perceived as the main source of both analytical and algorithmic sign perceptrons hardness, the natural choice for different activations leads towards allowing various continuous ones as well. Many of them have already found a steady place in NN architectures. Examples include but are not limited to ReLU, quadratic, tanh, erf and so on. Such activations make things a bit easier and, consequently, a little bit more is known about their capacities. For example, \cite{Yama93} suggested for deep nets and  \cite{GBHuang03} proved for 4-layer nets that the capacity is at least $O(w)$ for sigmoids. \cite{ZBHRV17,HardrtMa16} showed similar results for ReLU while additionally restricting on the number of nodes. Such a restriction though was later on removed in \cite{YunSuJad19} for both tanh and ReLU.

\underline{\emph{Statistical physics -- Replica theory:}} Given that the \emph{precise} capacity characterizations (as functions of the number of the hidden layer neurons $d$) do not allow for any \emph{qualitative/scaling} descriptions (say, of the $O(\cdot)$ type), the mathematically rigorous results are often very hard to achieve. In such situations, statistical physics replica methods are an excellent (and often irreplaceable) tool to produce, non-rigorous, but expectedly  \emph{precise} analyses.  As the hardness of the precise analytical studying of various NN features has been recognized in the mid-eighties of the last century, the application of the replica methods in capacity characterization has been around for close to four decades. The foundational concepts of such an approach were laid out in the pioneering works \cite{GarDer88,Gar88}, whre various forms of single perceptrons were discussed. Here, we focus more on the ensuing ones that relate to the network architectures. In particular, \cite{EKTVZ92,BHS92} studied the very same, TCM architecture, as we do (as well as the above mentioned, related, FCM one). For the sign perceptrons, they  obtained the closed form replica symmetry based capacity predictions for any number of the neurons in the hidden layer, $d$. They established the corresponding large $d$ scaling behavior. Both of these results were proven as mathematically rigorous capacity upper bounds in  \cite{Stojnictcmspnncapliftedrdt23,Stojnictcmspnncaprdt23}. Moreover, \cite{EKTVZ92,BHS92} showed that their replica symmetry based large $d$ predictions violate the mathematically rigorous ones obtained through the uniform-bounding extension of \cite{Cover65,Winder,Winder61,Wendel62} given in \cite{MitchDurb89}. This contradiction was remedied in \cite{EKTVZ92,BHS92} by studying the first level of replica symmetry breaking (rsb) and showing that it lowers the capacity. Related large $d$ scaling rsb considerations were also presented in \cite{MonZech95} for both the committee and the so-called parity machines (PM) (more on the earlier PM replica considerations can be found in, e.g., \cite{BarKan91,BHK90}). Also, for the FCM architecture, a bit later, \cite{Urban97,XiongKwonOh97} obtained the large $d$ scaling that matches the upper-bounding one of \cite{MitchDurb89}. Particularly relevant to our work are two very recent lines of work. \cite{BalMalZech19} first obtained the first level of rsb capacity predictions for the TCM architectures with the ReLU activations and \cite{ZavPeh21} moved things even further and obtained similar rsb predictions for several different activations, including the well known linear, ReLU, erf, quadratic, and tanh. A key difference with respect to our results should also be noted. Namely, both sets of results, \cite{BalMalZech19} and \cite{ZavPeh21}, relate to the networks with large $d$ (basically, to the networks with $d\rightarrow\infty$), whereas our results are obtained for any given (even) $d$.

\underline{\emph{Practical achievability:}} Another line of work attracted a lot of interest over the last several years and should be mentioned as well. It is related to the design and analysis of efficient algorithms that can be used to train the networks to potentially approach the capacity. The key focus has been on showing that the simple gradient based methods might actually perform quite well in this context. The so-called mild over-parametrization (moderately larger number of free parameters, $w$, compared to the size of the data set, $m$) is deemed as sufficing to ensure excellent performance of gradient based methods. More on the recent progress in these directions can be found in, e.g., \cite{DuZhaiPoc18,GeWangZhao19,ADHLW19,JiTel19,LiLiang18,OymSol19,RuoyuSun19,SongYang19,ZCZG18}. These results mostly relate to FCMs but are also extendable to TCMs as well.

\subsection{Contributions}
\label{sec:contrib}

The main object of our study is the so-called $n$-scaled memory capacity of the TCM NNs with various (different from standard \emph{sign} one) activation functions in the hidden layer. In other words, we study
\begin{equation}\label{eq:model4}
c(d;\f^{(2)})\triangleq\lim_{n\rightarrow\infty} \frac{C(A([n,d,1];\f^{(2)}))}{n}.
\end{equation}
A very strong progress in characterizing $C(A([n,d,1];\mbox{sign}))$ for any given (odd) $d$ has been made in \cite{Stojnictcmspnncaprdt23}. In particular, utilizing the powerful Random Duality Theory (RDT) mathematical engine, \cite{Stojnictcmspnncaprdt23} provides an explicit upper bound $\hat{c}(d;\mbox{sign})$ on $c(d;\mbox{sign})$. Numerical results obtained for smaller values of $d$ suggested a strong benefit in adding more neurons in a network architecture context. On the other hand, we, in this paper, make a substantial progress in several different aspects including both methodological and practical ones.

\underline{\emph{A summary of the main technical results of the paper:}}  \emph{\textbf{(i)}} We first show that the main framework from \cite{Stojnictcmspnncaprdt23,Stojnictcmspnncapliftedrdt23} (established relying on the RDT and pl RDt principles) for the capacity analysis of sign perceptron networks can be utilized for different activations as well. \emph{\textbf{(ii)}} To produce concrete capacity estimates, we then focus on several particular activations for which neat and convenient results can be obtained. We fist start with the \emph{linear} activations and show that the capacity of the TCM network with one hidden layer network is identical to the single spherical perceptron. We then switch to the \emph{quadratic} activations and obtain the close form analytical RDT based capacity upper bounds and their refined pl RDT counterparts. Finally, we consider the \emph{ReLU} activations and obtain the corresponding analytical closed form RDT and pl RDT results. \emph{\textbf{(iii)}} For all activations, we conduct the needed numerical evaluations to obtain the concrete capacity values as well. \emph{\textbf{(iv)}} All our results are obtained for an extremely challenging scenario where the number of the hidden layer neurons, $d$, can be any even positive integer. This allows us to uncover two rather fascinating phenomena: 1) Both thr RDT and the pl RDT estimates are decreasing for large $d$ while converging to a \emph{constant} (not dependent on $d$) value; and 2) The maxima of both estimates for both quadratic and ReLU activations are achieved for $d=2$ neurons in the hidden layer. This is a bit contrary to the common wisdom and in a strike contrast with the corresponding behavior of the sign activations where the capacity estimates grow with $d$. In particular, one effectively has that when it comes to the memory capabilities, uncontrollably increasing the width of the hidden layer may not always be among the most recommended architecture building strategies.

We show some of the concrete capacity estimates that we obtained for all the three mentioned activations, linear, quadratic, and ReLU, in Table \ref{tab:tab1}. In Figure \ref{fig:fig1} we also visualize the quadratic ones as well. Few key smallest values of $d$ are shown explicitly in the table and a much wider range of $d$ is shown in the figure (it goes without saying that the quadratic activation for $d=1$ does not make sense). For the completeness, we, in Figure \ref{fig:fig1}, also show the $d\rightarrow\infty$ asymptotics obtained for quadratic activations utilizing the replica methods in \cite{ZavPeh21}.

 \begin{table}[h]
  \caption{\textbf{\prp{Theoretical estimates}} of the $\f^{(2)}$-activated hidden layer TCM capacity upper bounds}
  \label{tab:tab1}
  \centering
  \begin{tabular}{cccccc}
    \hline\hline
  \textbf{Activation} &  \textbf{Upper bound on} & \textbf{Methodology}  & \multicolumn{3}{c}{$d$}                   \\
    \cline{4-6}
(function) &     $c(d;\f^{(2)})\triangleq\lim_{n\rightarrow\infty} \frac{C(A([n,d,1];\f^{(2)}))}{n}$ \hspace{-.1in} &     & $\mathbf{1}$   & $\mathbf{2}$     & $\mathbf{4}$ \\
    \hline\hline
  \textbf{\emph{linear}} &   $\bar{c}(d;\f^{(2)})$ & $ $  \hspace{.2in}  \bl{\textbf{RDT}} \hspace{.2in} $ $   & \bl{$\mathbf{2}$} & \bl{$\mathbf{2}$}  & \bl{$\mathbf{2}$}     \\
\cline{2-6}
  ($\f^{(2)}(\x)=\x$)  &  $ $  \hspace{.2in} $\hat{c}(d;\f^{(2)})$ \hspace{.2in} $ $ &  \prp{\textbf{pl RDT}}   & \prp{$\mathbf{2}$} & \prp{$\mathbf{2}$}  & \prp{$\mathbf{2}$}  \\
    \hline\hline
  \textbf{\emph{quadratic}} &   $\bar{c}(d;\f^{(2)})$ & $ $  \hspace{.2in}  \bl{\textbf{RDT}} \hspace{.2in} $ $   & \bl{$\mathbf{-}$} & \bl{$\mathbf{5.498}$}  & \bl{$\mathbf{4.660}$}     \\
\cline{2-6}
 ($\f^{(2)}(\x)=\x^2$)   &  $ $  \hspace{.2in} $\hat{c}(d;\f^{(2)})$ \hspace{.2in} $ $ &  \prp{\textbf{pl RDT}}   & \prp{$\mathbf{-}$} & \prp{$\mathbf{4.065}$}  & \prp{$\mathbf{3.657}$}  \\
    \hline\hline
  \textbf{\emph{ReLU}} &   $\bar{c}(d;\f^{(2)})$ & $ $  \hspace{.2in}  \bl{\textbf{RDT}} \hspace{.2in} $ $   & \bl{$\mathbf{2}$} & \bl{$\mathbf{3.810}$}  & \bl{$\mathbf{3.066}$}     \\
\cline{2-6}
($\f^{(2)}(\x)=\max (\x,0)$)   &  $ $  \hspace{.2in} $\hat{c}(d;\f^{(2)})$ \hspace{.2in} $ $ &  \prp{\textbf{pl RDT}}   & \prp{$\mathbf{2}$} & \prp{$\mathbf{3.810}$}  & \prp{$\mathbf{3.066}$}  \\
       \hline\hline
  \end{tabular}
\end{table}

\begin{figure}[h]
\centering
\centerline{\includegraphics[width=1\linewidth]{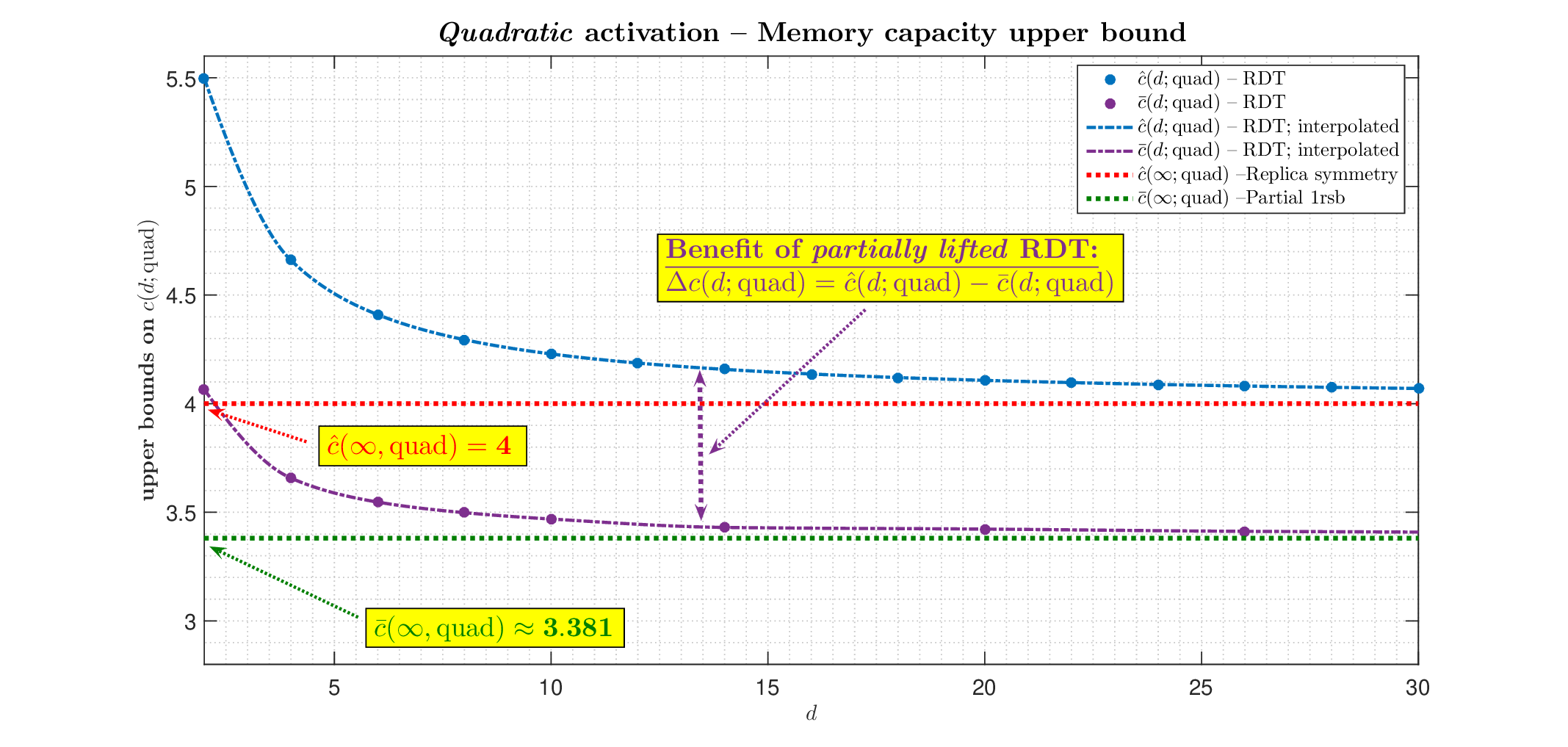}}
\caption{Memory capacity upper bound as a function of the number of neurons, $d$, in the hidden layer; 1-hidden layer TCM with \textbf{\emph{quadratic}} activations; \bl{\textbf{\emph{plain}  RDT}} versus \prp{\textbf{\emph{partially lifted} RDT}} (\red{\textbf{Replica symmetry (RS)}} and \dgr{\textbf{Partial 1rsb}} $d\rightarrow\infty$ estimates are included as well)}
\label{fig:fig1}
\end{figure}

\section{Algebraic description of network data processing}
\label{sec:amathdataproc}

To put everything on the right mathematical track and facilitate writing and overall presentation,  we first set $W\triangleq W^{(2)}$. After recalling that $W^{(1)}=I$ and $W^{(3)}=\w^T$, we then have for any $k\in\{1,2,\dots,m\}$
\begin{equation}\label{eq:ta1}
\x^{(1)}=\x^{(0,k)}  \quad \Longrightarrow \quad  \x^{(2)}=\f^{(1)}(W^{(1)}\x^{(1)})=\f^{(1)}(\x^{(1)})=\x^{(1)}=\x^{(0,k)},
\end{equation}
and
\begin{equation}\label{eq:ta2}
\x^{(2)}=\x^{(0,k)}  \quad \Longrightarrow \quad  \x^{(3)}=\f^{(2)}(W^{(2)}\x^{(2)})=\f^{(2)}(W\x^{(0,k)}),
\end{equation}
and
\begin{equation}\label{eq:ta3}
 \x^{(3)}=\f^{(2)}(W\x^{(0,k)}) \quad \Longrightarrow \quad  \x^{(4)}=\f^{(3)}(W^{(3)}\x^{(3)})=\f^{(3)}(\w^T\x^{(3)})
 =\mbox{sign}(\w^T  \f^{(2)}(W\x^{(0,k)})).
\end{equation}
A very neat closed-form explicit relation between the input and the output of the network can be obtained by connecting beginning in (\ref{eq:ta1}) and end in (\ref{eq:ta3})
\begin{equation}\label{eq:ta4}
\x^{(1)}=\x^{(0,k)} \quad \Longrightarrow \quad  \x^{(4)} =\mbox{sign}(\w^T  \f^{(2)}(W\x^{(0,k)})).
\end{equation}
For network to operate properly, it is then both necessary and sufficient that the following condition holds
\begin{equation}\label{eq:ta5}
\y^{(0,k)}=\mbox{sign}(\w^T  \f^{(2)}(W\x^{(0,k)})).
\end{equation}
After setting
\begin{equation}\label{eq:ta6}
\y\triangleq \begin{bmatrix}
               \y^{(0,1)} & \y^{(0,2)} & \dots & \y^{(0,m)}
            \end{bmatrix}^T   \qquad \mbox{and} \qquad X\triangleq \begin{bmatrix}
               \x^{(0,1)} & \x^{(0,2)} & \dots & \x^{(0,m)}
            \end{bmatrix}^T,
\end{equation}
one can then rewrite (\ref{eq:ta5}) in generic matrix form as
\begin{equation}\label{eq:ta7}
\left (\exists W\in\mR^{d\times n}| \|\y^T-\mbox{sign}(\w^T  \f^{(2)}(WX^T))\|_2=0 \right )  \quad \Longleftrightarrow \quad \left ( \left ( X,\y \right ) \mbox{is memorized} \right ),
\end{equation}
and the $k$-th data pair $\left ( \x^{(0,k)},\y^{(0,k)} \right )$ are the $k$-th rows of $m\times n$ matrix $X$ and $m\times 1$ column vector $\y$. This further leads to the following (effectively an alternative to (\ref{eq:ta7})):
\begin{center}
 	\tcbset{beamer,sidebyside,lower separated=false, fonttitle=\bfseries, coltext=black,
		interior style={top color=yellow!20!white, bottom color=yellow!60!white},title style={left color=black, right color=red!50!blue!60!white},
		width=(\linewidth-4pt)/4,before=,after=\hfill,fonttitle=\bfseries,equal height group=AT}
 	\begin{tcolorbox}[title=Algebraic memorization characterization of the $\f^{(2)}$-activated hidden layer TCMs:,sidebyside,width=1\linewidth]
\vspace{-.15in}\begin{eqnarray}\label{eq:ta8}
\hspace{-.3in} 0=\xi\triangleq \min_{W,Q} & & \hspace{-.1in}\|\y-\mbox{sign}(\f^{(2)}(Q) \w)\|_2 \nonumber \\
\hspace{-.5in} \mbox{subject to} & & \hspace{-.1in}XW^T=Q
\end{eqnarray}
 \tcblower
 \hspace{-.2in}$\Longleftrightarrow$ \hspace{.1in} Data set $\left (X,\y \right )$ is properly memorized.
 		\vspace{-.0in}
 	\end{tcolorbox}
\end{center}
Clearly, the above algebraic formulation is the key mathematical problem on the path towards ensuring proper network data memorization. We therefore analyze it below in more detail.

\section{Random Duality Theory (RDT) based capacity analysis}
\label{sec:ubrdt}

As mentioned earlier, we consider statistical data sets with elements of $X$ being iid standard normals. Due to rotational symmetry, one can then, without a loss of generality, assume that the elements of $\y$ are all equal to 1, i.e. one can assume that $\y=\1$. The above key optimization can then be rewritten
\begin{eqnarray}\label{eq:ta9}
\xi=\min_{Z,Q} & & \|\1-\mbox{sign}(\f^{(2)}(Q) \w)\|_2 \nonumber \\
 \mbox{subject to} & & XZ=Q,
\end{eqnarray}
where a cosmetic change, $Z=W^T$, is introduced to facilitate writing. As in \cite{Stojnictcmspnncapliftedrdt23,Stojnictcmspnncaprdt23}, we here consider the TCM architecture, with a sparse $Z$ that ensures that treelike network architecture. This basically means that we consider $Z$ such that the only nonzero elements of its $j$-th column are in rows $\cS^{(j)}\triangleq\{(j-1)\delta+1,(j-1)\delta+2,\dots,j\delta\}$. Utilizing this $Z$ specialization and given that the above problem is insensitive with respect to the
scaling of $Z$ or $Q$, one can write
\begin{eqnarray}\label{eq:ta9a}
\xi=\min_{Z,Q} & & \|\1-\mbox{sign}(\f^{(2)}(Q) \w)\|_2 \nonumber \\
 \mbox{subject to} & & XZ=Q\nonumber \\
  & & \|Z_{:,j}\|_2=1 \nonumber \\
  & & \mbox{supp}(Z_{:,j})=\cS^{(j)}, 1\leq j\leq d,
\end{eqnarray}
with $\|Z_{:,j}\|_2$ being the norm of the $j$-th column of $Z$. A further trivial rewriting of the above gives
\begin{eqnarray}\label{eq:ta9aa0}
\xi=\min_{\z^{(j)},Q} & & \|\1-\mbox{sign}(\f^{(2)}(Q) \w)\|_2 \nonumber \\
 \mbox{subject to} & & X^{(j)}\z^{(j)}=Q_{:,j}, 1\leq j\leq d, \nonumber \\
  & & \|\z^{(j)}\|_2=1 \nonumber \\
  & & \z^{(j)}\in\mR^{\delta}, Q\in\mR^{m\times d},
\end{eqnarray}
where $X^{(j)}=X_{:,\cS^{(j)}}\in\mR^{m\times \delta}$. We follow into the footsteps of \cite{Stojnictcmspnncapliftedrdt23,Stojnictcmspnncaprdt23} and to statistically analyze the optimizations in (\ref{eq:ta9a}) and (\ref{eq:ta9aa0}), we utilize the powerful mathematical engine called Random Duality Theory (RDT) developed in a long series of work \cite{StojnicCSetam09,StojnicICASSP10var,StojnicCSetamBlock09,StojnicICASSP10block,StojnicRegRndDlt10}. To make the presentation easier to follow, we will try to parallel as closely as possible the approach presented in \cite{Stojnictcmspnncaprdt23}. However, to avoid unnecessary repetitions, we only briefly recall on some of the concepts that are identical or very similar to the corresponding ones of \cite{Stojnictcmspnncaprdt23}, and instead focus on key differences. As in \cite{Stojnictcmspnncaprdt23}, we start by first summarizing the RDT main principles and then continue by discussing each of them within the context of our interest here.

\vspace{-.0in}\begin{center}
 	\tcbset{beamer,lower separated=false, fonttitle=\bfseries, coltext=black ,
		interior style={top color=yellow!20!white, bottom color=yellow!60!white},title style={left color=black!80!purple!60!cyan, right color=yellow!80!white},
		width=(\linewidth-4pt)/4,before=,after=\hfill,fonttitle=\bfseries}
 \begin{tcolorbox}[beamer,title={\small Summary of the RDT's main principles} \cite{StojnicCSetam09,StojnicRegRndDlt10}, width=1\linewidth]
\vspace{-.15in}
{\small \begin{eqnarray*}
 \begin{array}{ll}
\hspace{-.19in} \mbox{1) \emph{Finding underlying optimization algebraic representation}}
 & \hspace{-.0in} \mbox{2) \emph{Determining the random dual}} \\
\hspace{-.19in} \mbox{3) \emph{Handling the random dual}} &
 \hspace{-.0in} \mbox{4) \emph{Double-checking strong random duality.}}
 \end{array}
  \end{eqnarray*}}
\vspace{-.2in}
 \end{tcolorbox}
\end{center}\vspace{-.0in}

As in \cite{Stojnictcmspnncaprdt23}, all the key results (including both simple to more complicated ones) are formulated as lemmas and theorems.

\vspace{.1in}

\noindent \underline{1) \textbf{\emph{Algebraic memorization characterization:}}}  The following lemma summarizes the above algebraic discussion by providing a precise resulting optimization representation of the network memorization property
 and is a mirrored analogue to Lemma 1 in \cite{Stojnictcmspnncaprdt23}.
\begin{lemma}(Algebraic optimization representation)
Assume a 1-hidden layer TCM with architecture $A([n,d,1];\f^{(2)})$. Any given data set $\left (\x^{(0,k)},1\right )_{k=1:m}$ can not be properly memorized by the network if
\begin{equation}\label{eq:ta10}
  f_{rp}(X)>0,
\end{equation}
where
\begin{equation}\label{eq:ta11}
f_{rp}(X)\triangleq \frac{1}{\sqrt{n}}\min_{\|\z^{(j)}\|_2=1,Q} \max_{\Lambda\in\mR^{m\times d}} \|\1-\mbox{\emph{sign}}(\f^{(2)}(Q) \w)\|_2 +\sum_{j=1}^{d}(\Lambda_{:,j})^TX^{(j)}\z^{(j)} -\tr(\Lambda^TQ),
\end{equation}
and $X\triangleq \begin{bmatrix}
               \x^{(0,1)} & \x^{(0,2)} & \dots & \x^{(0,m)}
            \end{bmatrix}^T$.
  \label{lemma:lemma1}
\end{lemma}
\begin{proof}
Immediate consequence of Lemma 1 in \cite{Stojnictcmspnncaprdt23}.
\end{proof}

Of our interest below is mathematically the most challenging, so-called \emph{linear}, regime with
\begin{equation}\label{eq:ta14}
  \alpha\triangleq \lim_{n\rightarrow\infty}\frac{m}{n}.
\end{equation}
The above lemma is of purely algebraic nature and as such it holds for any given data set $\left (\x^{(0,k)},1\right )_{k=1:m}$. To analyze (\ref{eq:ta10}) and (\ref{eq:ta11}), the RDT further proceeds by accounting for a statistical $X$.


\vspace{.1in}
\noindent \underline{2) \textbf{\emph{Determining the random dual:}}} We follow the standard RDT practice and utilize the so-called concentration of measure property, which basically means that for any fixed $\epsilon >0$,  we have (see, e.g. \cite{StojnicCSetam09,StojnicRegRndDlt10,StojnicICASSP10var,Stojnictcmspnncaprdt23,Stojnictcmspnncapliftedrdt23})
\begin{equation}
\lim_{n\rightarrow\infty}\mP_X\left (\frac{|f_{rp}(X)-\mE_X(f_{rp}(X))|}{\mE_X(f_{rp}(X))}>\epsilon\right )\longrightarrow 0.\label{eq:ta15}
\end{equation}
Moreover, the following so-called random dual theorem is another key ingredient of the RDT machinery and is a mirrored alternative to Theorem 1 from \cite{Stojnictcmspnncaprdt23}.
\begin{theorem}(Memorization characterization via random dual) Let $d$ be any even positive integer. Consider TCM with $d$ neurons in the hidden layer, and architecture  $A([n,d,1];\f^{(2)})$, and let the elements of $X\in\mR^{m\times n}$, $G\in\mR^{m\times d}$, and $H\in\mR^{\delta\times d}$ be iid standard normals. Set
\vspace{-.0in}
\begin{eqnarray}
\phi(Q) & \triangleq & \|\1-\mbox{\emph{sign}}(\f^{(2)}(Q) \w)\|_2\nonumber \\
 f_{rd}(G,H) & \triangleq &
 \frac{1}{\sqrt{n}}\min_{\phi(Q)=0,\|\z^{(j)}\|_2=1}\max_{\Lambda \in R^{m\times d},\|\Lambda\|_F=1}\lp \tr(\Lambda^T G)+\sum_{j=1}^{d}\|\Lambda_{:,j}\|_2(H_{:,j})^T\z^{(j)} -\tr(\Lambda^TQ) \rp \nonumber \\
 \phi_0 & \triangleq & \lim_{n\rightarrow\infty} \mE_{G,H}f_{rd}(G,H)  .\label{eq:thm1ta16}
\vspace{-.0in}\end{eqnarray}
One then has \vspace{-.02in}
\begin{eqnarray}
\hspace{-.3in}(\phi_0  > 0)   &  \Longrightarrow  & \lp \lim_{n\rightarrow\infty}\mP_{X}(f_{rd}>0)\longrightarrow 1\rp
\quad  \Longrightarrow \quad \lp \lim_{n\rightarrow\infty}\mP_{X}(f_{rp}>0)\longrightarrow 1 \rp  \nonumber \\
& \Longrightarrow & \lp \lim_{n\rightarrow\infty}\mP_{X}(A([n,d,1];\f^{(2)}) \quad \mbox{fails to memorize data set} \quad (X,\1))\longrightarrow 1\rp.\label{eq:thm1ta17}
\end{eqnarray}
\label{thm:thm1}
\end{theorem}\vspace{-.17in}
\begin{proof}
Immediate consequence of Theorem 1 in \cite{Stojnictcmspnncaprdt23}.
\end{proof}
%
%
%
%
%
\vspace{.1in}
\noindent \underline{3) \textbf{\emph{Handling the random dual:}}} Proceeding as in  \cite{Stojnictcmspnncaprdt23}, one first solves the inner maximization over $\Lambda$ and then the minimization over $\z^{(j)}$ to obtain for $f_{rd}(G,H)$ from (\ref{eq:thm1ta16})
\begin{eqnarray}
 f_{rd}(G,H) & = &    \frac{1}{\sqrt{n}}
 \min_{\phi(Q)=0}\sqrt{ \|G-Q\|_F^2-2\sum_{j=1}^{d} \|G_{:,j}-Q_{:,j}\|_2\|H_{:,j}\|_2 +\|H\|_F^2}.\label{eq:ta18}
\end{eqnarray}
This then further gives
\begin{eqnarray}
\phi_0 & = & \lim_{n\rightarrow \infty}\mE_{G,H} f_{rd}(G,H)  =  \lim_{n\rightarrow \infty}\mE_{G} \frac{1}{\sqrt{n}}
 \min_{\phi(Q)=0} \|G-Q\|_F- 1,\label{eq:ta18a0}
\end{eqnarray}
where, as discussed in  \cite{Stojnictcmspnncaprdt23}, the above equality obtained relying on the concentrations can be replaced by an inequality if one alternatively relies on the Cauchy-Schwartz  inequalities (both options are perfectly  sufficient for the subsequent analysis).

The remaining focus is on the residual optimization over $Q$. To that end we set
\begin{equation}\label{eq:ta19}
  \phi_i(Q_{i,1:d})\triangleq  \mbox{sign}(\f^{(2)}(Q_{i,1:d}) \w),
\end{equation}
and further write
\begin{eqnarray}\label{eq:ta20}
   \min_{\phi(Q)=0} \|G-Q\|_F &=& \sqrt{\min_{\phi(Q)=0} \sum_{i=1}^{m}\sum_{j=1}^{d}(G_{ij}-Q_{ij})^2}  =  \sqrt{\sum_{i=1}^{m} \min_{\phi_i(Q_{i,1:d})=1}  \sum_{j=1}^{d}(G_{ij}-Q_{ij})^2} \nonumber \\
   &=& \sqrt{\sum_{i=1}^{m}  z_i(G_{i,1:d})},
\end{eqnarray}
with
\begin{equation}\label{eq:ta21}
  z_i(G_{i,1:d})\triangleq \min_{\phi_i(Q_{i,1:d})=1}  \sum_{j=1}^{d}(G_{ij}-Q_{ij})^2.
\end{equation}
To further facilitate writing and remove unnecessary notations, we will set
\begin{eqnarray}\label{eq:lin2}
 \g & \triangleq  &  G_{i,1:d}^T \nonumber \\
 \q & \triangleq  &  Q_{i,1:d}^T,
\end{eqnarray}
and
\begin{equation}\label{eq:ta21defzi}
  z_i(\g;\f^{(2)}) \triangleq  z_i(G_{i,1:d})= \min_{\phi_i(Q_{i,1:d})=1}  \sum_{j=1}^{d}(G_{ij}-Q_{ij})^2
  =\min_{\f^{(2)}(Q_{i,1:d}) \w\geq 0}  \sum_{j=1}^{d}(G_{ij}-Q_{ij})^2 =\min_{\f^{(2)}(\q^T) \w\geq 0}  \|\g-\q\|_2^2.
\end{equation}
The above mechanism is generic and in principle applies to any type of activation $\f^{(2)}$. To obtain concrete capacity estimates, we below proceed by considering several particular activation examples that have attracted a significant attention in NN literature.

\subsection{Different $\f^{(2)}$ activations}
\label{sec:diffact}

We focus on three well known $\f^{(2)}$ activations: i) \emph{linear}, ii) \emph{quadratic}, and iii) \emph{ReLU}. For each of them we obtain relatively  elegant final capacity bounding estimates.

\subsubsection{Linear hidden layer activations -- $\f^{(2)}(\x)=\x$}
\label{sec:lin}

Since the linear function is odd (i.e., since $f^{(2)}(-\x)=-\x=-\f^{(2)}(\x)$), we can, without a loss of generality, assume that, say, $\w=\1$, where $\1$ is the column vector  of appropriate dimension with all of its components equal to one. However, as the analysis below shows, linearity is a very particular form of activation where such an assumption is actually not needed. In fact, any $\w$ suffices. This should be kept in mind for later on when we study other two types of activations where $\w$ will have to take particular forms to ensure that network functioning actually makes sense at all. From (\ref{eq:ta19})
 and (\ref{eq:ta21}), we then recognize the key optimization problem of interest
\begin{equation}\label{eq:lin1}
  z_i(G_{i,1:d}) =\min_{\mbox{sign}(\f^{(2)}(Q_{i,1:d}) \w)=1}  \sum_{j=1}^{d}(G_{ij}-Q_{ij})^2
  =\min_{\f^{(2)}(Q_{i,1:d}) \w\geq 0}  \sum_{j=1}^{d}(G_{ij}-Q_{ij})^2.
\end{equation}
For the linear activation, one can then rewrite (\ref{eq:lin1}) as
\begin{equation}\label{eq:lin3}
  z_i(\g;\mbox{lin})  =  z_i(G_{i,1:d})  =\min_{\f^{(2)}(Q_{i,1:d}) \w\geq 0}  \sum_{j=1}^{d}(G_{ij}-Q_{ij})^2
  =\min_{\f^{(2)}(\q^T) \w\geq 0}  \|\g-\q\|_2^2  =\min_{\q^T\w\geq 0}  \|\g-\q\|_2^2.
\end{equation}
Proceeding by writing Lagrangian, we then further have
\begin{equation}\label{eq:lin4}
  z_i(\g;\mbox{lin})  =\min_{\q^T\w\geq 0}  \|\g-\q\|_2^2
 =\min_{\q} \max_{\lambda\geq 0} \|\g-\q\|_2^2 -\lambda(\q^T\w).
\end{equation}
Relying on the strong duality, one then also finds
\begin{equation}\label{eq:lin5}
  z_i(\g;\mbox{lin})    =\min_{\q} \max_{\lambda\geq 0} \|\g-\q\|_2^2 -\lambda(\q^T\w)
   =\max_{\lambda\geq 0} \min_{\q} \|\g-\q\|_2^2 -\lambda(\q^T\w).
\end{equation}
To solve the inner optimization over $\q$, we then consider the following derivative
\begin{equation}\label{eq:lin6}
 \frac{d\lp \|\g-\q\|_2^2 -\lambda(\q^T\w)\rp}{d\q}=2(\q-\g)+\lambda\w.
\end{equation}
After equalling the above derivative to zero one then finds
\begin{equation}\label{eq:lin7}
 \q=\g-\frac{\lambda\w}{2}.
\end{equation}
Plugging back this value in the objective in (\ref{eq:lin5}), one obtains
\begin{equation}\label{eq:lin8}
  z_i(\g;\mbox{lin})    =\max_{\lambda\geq 0} \min_{\q}  \|\g-\q\|_2^2 -\lambda(\q^T\w)
  =\max_{\lambda\geq 0} -\frac{\lambda^2\|\w\|_2^2}{4}-\lambda\g^T\w.
\end{equation}
Optimizing further over $\lambda$, one finds
\begin{equation}\label{eq:lin9}
\lambda_{opt}=\max\lp-\frac{2\g^T\w}{\|\w\|_2^2},0\rp.
\end{equation}
After plugging $\lambda_{opt}$ back in (\ref{eq:lin8}), one obtains
\begin{equation}\label{eq:lin10}
  z_i(\g;\mbox{lin}) = \frac{\lp\max\lp-\g^T\w,0\rp\rp^2}{\|\w\|_2^2}= \lp\max\lp-\g^T\frac{\w}{\|\w\|_2},0\rp\rp^2= \lp\max\lp g_i,0\rp\rp^2,
\end{equation}
where $g_i$ is a standard normal random variable. Connecting (\ref{eq:ta18a0}), (\ref{eq:ta20}), (\ref{eq:lin3}), and (\ref{eq:lin10}), we finally have
\begin{eqnarray}
\phi_0 & = &   \lim_{n\rightarrow \infty}\mE_{G} \frac{1}{\sqrt{n}}
 \min_{\phi(Q)=0} \|G-Q\|_F- 1 \nonumber \\
 & = &   \lim_{n\rightarrow \infty}\mE_{G} \frac{1}{\sqrt{n}}
 \sqrt{\sum_{i=1}^{m}  z_i(G_{i,1:d})}- 1 \nonumber \\
 & = &    \sqrt{\alpha \mE  z_i(\g;\mbox{lin}) }- 1 \nonumber \\
 & = &    \sqrt{\alpha \mE \lp\max\lp g_i,0\rp\rp^2}- 1 \nonumber \\
 & = &    \sqrt{\frac{\alpha}{2}}- 1.\label{eq:lin11}
\end{eqnarray}

 We summarize the above results in the following lemma.
\begin{lemma}(Memory capacity; \textbf{linear} activation) Assume the setup of Theorem \ref{thm:thm1}. For linear $\f^{(2)}(\x)=\x$, let $c(d;\mbox{lin})\triangleq c(d;\f^{(2)}(\x)=\x)$ be the $n$-scaled memory capacity from (\ref{eq:model4}). One then has the following for such
\vspace{-.0in}
\vspace{-.0in}\begin{center}
\tcbset{beamer,lower separated=false, fonttitle=\bfseries,
coltext=black , interior style={top color=orange!10!yellow!30!white, bottom color=yellow!80!yellow!50!white}, title style={left color=orange!10!cyan!30!blue, right color=green!70!blue!20!black}}
 \begin{tcolorbox}[beamer,title=\textbf{($n$-scaled) memory capacity:},lower separated=false, fonttitle=\bfseries,width=.42\linewidth] 
\vspace{-.15in}
 \begin{eqnarray*}
\hspace{-.0in} c(d;\mbox{lin})=  \mathbf{2}. \end{eqnarray*}
\vspace{-.15in}
 \end{tcolorbox}
\end{center}\vspace{-.0in}
If and only if the sample complexity $m$ is such that $\alpha\triangleq \lim_{n\rightarrow\infty}\frac{m}{n}\geq c(d;\mbox{lin})$ then
\begin{eqnarray}
 \lim_{n\rightarrow\infty}\mP_{X}(A([n,d,1];\mbox{sign}) \quad \mbox{fails to memorize data set} \quad (X,\1))\longrightarrow 1.\label{eq:lemma2ta30}
\end{eqnarray}
 \label{lemma:lemma2}
\end{lemma}\vspace{-.17in}
\begin{proof}
Follows immediately from the above discussion.
\end{proof}
The above lemma effectively states that only when the sample complexity $m$ is such that $m>2n$ (with $n$ being the data vectors' ambient dimension) the network fails to memorize the data. Consequently, one has for the memory capacity of the $d$ hidden layer \emph{linearly} activated neurons TCMs, $C(A([n,d,1];\mbox{lin}))=2n$. This shows that the capacity does not change as the width of the hidden layer increases. Moreover, it shows that the capacity remains equal to the capacity of the single spherical perceptron neuron (see, e.g.,
\cite{Schlafli,Cover65,Winder,Winder61,Wendel62,Cameron60,Joseph60,BalVen87,Ven86,DT,StojnicISIT2010binary,DonTan09Univ,DTbern,Gar88,StojnicGardGen13,StojnicGardSphErr13}).

\noindent \underline{4) \textbf{\emph{Double checking the strong random duality:}}} Strictly speaking the above analysis establishes the upper bound on the capacity. However, given the presence of the underlying convexity and the strong deterministic duality, the machinery of \cite{StojnicRegRndDlt10,StojnicUpper10,StojnicGorEx10} ensures that the strong random duality is in place as well which then implies that the established upper bounds are in fact tight.

\subsubsection{Quadratic hidden layer activations -- $\f^{(2)}(\x)=\x^2$}
\label{sec:quad}

Since the quadratic function $\f^{(2)}(\x)=\x^2\geq 0$ one needs to carefully make a choice for vector $\w$ which ensures that the network functioning is of any use. Clearly, some of the components of $\w$  must be negative. Given the symmetry of the quadratic function, a natural choice that is typically considered in the literature in such situations is $\w$ with $\frac{d}{2}$ 1s and $\frac{d}{2}$ -1s. For the concreteness, we set
\begin{equation}\label{eq:quad00}
\w=\begin{bmatrix}
     -\1_{\frac{d}{2}\times 1} \\ \1_{\frac{d}{2}\times 1}
   \end{bmatrix},
\end{equation}
and to facilitate exposition avoid dimensional subscripts and simply write
 \begin{equation}\label{eq:quad0}
\w=\begin{bmatrix}
     -\1 \\ \1
   \end{bmatrix},
\end{equation}
assuming that the size of the column vectors $\1$ is $\frac{d}{2}\times 1$. As earlier, relying on (\ref{eq:ta19}), (\ref{eq:ta21}), (\ref{eq:lin2}), and (\ref{eq:lin1}), we then recognize the following key optimization problem of interest
 \begin{equation}\label{eq:quad3}
  z_i(\g;\mbox{quad})  =  z_i(G_{i,1:d})  =\min_{\f^{(2)}(Q_{i,1:d}) \w\geq 0}  \sum_{j=1}^{d}(G_{ij}-Q_{ij})^2
  =\min_{\f^{(2)}(\q^T) \w\geq 0}  \|\g-\q\|_2^2  =\min_{(\q^2)^T\w\geq 0}  \|\g-\q\|_2^2.
\end{equation}
After effectively splitting the problem into two parts, an interesting formulation can be obtained
 \begin{eqnarray}\label{eq:quad3a1}
  z_i(\g;\mbox{quad})   =  \min_{\q} & &   \sum_{i=1}^{\frac{d}{2}}(\g_i-\q_i)^2+ \sum_{i=\frac{d}{2}+1}^d(\g_i-\q_i)^2 \nonumber \\
\mbox{subject to}   & &  \sum_{i=1}^{\frac{d}{2}} \q_i^2\leq  \sum_{i=\frac{d}{2}+1}^d \q_i^2.
\end{eqnarray}
For a moment, we find it convenient to set $\sum_{i=\frac{d}{2}+1}^d \q_i^2=b$, define $\g^{(r)}\triangleq \g_{\frac{d}{2}+1:d}$, and look at the following optimization problem
 \begin{eqnarray}\label{eq:quad3a2}
  z_i^{(r)}(\g^{(r)},b)   =  \min_{\q} & &   \sum_{i=\frac{d}{2}+1}^d(\g_i-\q_i)^2 \nonumber \\
\mbox{subject to}   & &  \sum_{i=\frac{d}{2}+1}^d \q_i^2=b.
\end{eqnarray}
One can then trivially rewrite (\ref{eq:quad3a2}) as
 \begin{eqnarray}\label{eq:quad3a3}
  z_i^{(r)}(\g^{(r)},b)   =  \min_{\q} & &   \sum_{i=\frac{d}{2}+1}^d\g_i^2-2\sum_{i=\frac{d}{2}+1}^d\g_i\q_i +b \nonumber \\
\mbox{subject to}   & &  \sum_{i=\frac{d}{2}+1}^d \q_i^2=b.
\end{eqnarray}
Solving (\ref{eq:quad3a3}) then gives
 \begin{eqnarray}\label{eq:quad3a4}
  z_i^{(r)}(\g^{(r)},b) & = & \lp\sum_{i=\frac{d}{2}+1}^d\g_i^2-2\sqrt{b}\sqrt{\sum_{i=\frac{d}{2}+1}^d\g_i^2} +b\rp \nonumber \\
 & = & \lp\sum_{i=\frac{d}{2}+1}^d\g_i^2-2\sqrt{b}\sqrt{\sum_{i=\frac{d}{2}+1}^d\g_i^2} +b\rp  \nonumber \\
 & = & \lp\sqrt{\sum_{i=\frac{d}{2}+1}^d\g_i^2} -\sqrt{b}\rp^2.
\end{eqnarray}
One can then rewrite (\ref{eq:quad3a1}) as
 \begin{eqnarray}\label{eq:quad3a5}
  z_i(\g;\mbox{quad})   =  \min_{\q,b} & &   \sum_{i=1}^{\frac{d}{2}}(\g_i-\q_i)^2+   z_i^{(r)}(\g^{(r)},b) \nonumber \\
\mbox{subject to}   & &  \sum_{i=1}^{\frac{d}{2}} \q_i^2\leq  b.
\end{eqnarray}
A few additional algebraic transformations give
 \begin{eqnarray}\label{eq:quad3a6}
  z_i(\g;\mbox{quad})   =  \min_{\q,b} & &   \sum_{i=1}^{\frac{d}{2}}\g_i^2-2\sum_{i=1}^{\frac{d}{2}}\g_i\q_i + \sum_{i=1}^{\frac{d}{2}}\q_i^2+   z_i^{(r)}(\g^{(r)},b) \nonumber \\
\mbox{subject to}   & &  \sum_{i=1}^{\frac{d}{2}} \q_i^2\leq  b.
\end{eqnarray}
We again for a moment set $\sum_{i=}^{\frac{d}{2}} \q_i^2=b_1\leq b$, define $\g^{(l)}\triangleq \g_{1:\frac{d}{2}}$, and find
 \begin{eqnarray}\label{eq:quad3a6}
  z_i(\g;\mbox{quad})   =  \min_{b_1\leq b} & &   z_i^{(l)}(\g^{(l)},b_1)+   z_i^{(r)}(\g^{(r)},b),
\end{eqnarray}
where, analogously to (\ref{eq:quad3a4}), we also have
 \begin{eqnarray}\label{eq:quad3a7}
  z_i^{(l)}(\g^{(l)},b)   & = & \lp\sqrt{\sum_{i=1}^{\frac{d}{2}}\g_i^2} -\sqrt{b_1}\rp^2.
\end{eqnarray}
Clearly,
 \begin{eqnarray}\label{eq:quad3a7a1}
\sqrt{\sum_{i=1}^{\frac{d}{2}}\g_i^2}\leq \sqrt{\sum_{i=\frac{d}{2}+1}^d\g_i^2} \qquad \implies \qquad  z_i(\g)=0.
\end{eqnarray}
On the other hand, if $\sqrt{\sum_{i=1}^{\frac{d}{2}}\g_i^2}> \sqrt{\sum_{i=\frac{d}{2}+1}^d\g_i^2}$ then $b_1=b$ and the optimization from (\ref{eq:quad3a6}) becomes
 \begin{eqnarray}\label{eq:quad3a8}
  z_i(\g;\mbox{quad})   =  \min_{b} \lp\lp\sqrt{\sum_{i=1}^{\frac{d}{2}}\g_i^2} -\sqrt{b}\rp^2+ \lp\sqrt{\sum_{i=\frac{d}{2}+1}^d\g_i^2} -\sqrt{b}\rp^2 \rp.
\end{eqnarray}
Optimizing over $b$ one then finds
 \begin{eqnarray}\label{eq:quad3a9}
 \sqrt{b}=\frac{\sqrt{\sum_{i=1}^{\frac{d}{2}}\g_i^2}+ \sqrt{\sum_{i=\frac{d}{2}+1}^d\g_i^2}}{2},
\end{eqnarray}
and
 \begin{eqnarray}\label{eq:quad3a10}
\sqrt{\sum_{i=1}^{\frac{d}{2}}\g_i^2}> \sqrt{\sum_{i=\frac{d}{2}+1}^d\g_i^2} \qquad \implies \qquad   z_i(\g;\mbox{quad})   =  \frac{\lp\sqrt{\sum_{i=1}^{\frac{d}{2}}\g_i^2}- \sqrt{\sum_{i=\frac{d}{2}+1}^d\g_i^2}\rp^2}{2}.
\end{eqnarray}
Writing (\ref{eq:quad3a7a1}) and (\ref{eq:quad3a10}) in a more compact form gives
 \begin{eqnarray}\label{eq:quad3a11}
  z_i(\g;\mbox{quad})   =  \frac{\lp \max\lp\sqrt{\sum_{i=1}^{\frac{d}{2}}\g_i^2}- \sqrt{\sum_{i=\frac{d}{2}+1}^d\g_i^2},0\rp\rp^2}{2}.
\end{eqnarray}
Moreover, one also has
 \begin{eqnarray}\label{eq:quad3a12}
  z_i(\g;\mbox{quad})   =  \frac{\lp \max\lp a_i^{(1)}-  a_i^{(2)},0\rp\rp^2}{2},
\end{eqnarray}
where $ a_i^{(1)}$ and $ a_i^{(2)}$ are independent chi random variables with $\frac{d}{2}$ degrees of freedom. Connecting (\ref{eq:ta18a0}), (\ref{eq:ta20}), (\ref{eq:quad3}), and (\ref{eq:quad3a12}), we finally have
\begin{eqnarray}
\phi_0 & = &   \lim_{n\rightarrow \infty}\mE_{G} \frac{1}{\sqrt{n}}
 \min_{\phi(Q)=0} \|G-Q\|_F- 1 \nonumber \\
 & = &   \lim_{n\rightarrow \infty}\mE_{G} \frac{1}{\sqrt{n}}
 \sqrt{\sum_{i=1}^{m}  z_i(G_{i,1:d})}- 1 \nonumber \\
 & = &    \sqrt{\alpha \mE z_i(\g;\mbox{quad}) }- 1\nonumber \\
  & = &    \sqrt{\alpha \mE \frac{\lp \max\lp a_i^{(1)}-  a_i^{(2)},0\rp\rp^2}{2} }- 1.\label{eq:quad11}
\end{eqnarray}
Given the pdf of the chi random variable with $\frac{d}{2}$ degrees of freedom
\begin{eqnarray}
f_{\chi}(a) = \frac{2^{1-\frac{d}{4}}}{\gamma(\frac{d}{4})}a^{\frac{d}{2}-1}e^{-\frac{a^2}{2}},\label{eq:quad12}
\end{eqnarray}
one can also find
\begin{eqnarray}
  \mE \lp \max\lp a_i^{(1)}-  a_i^{(2)},0\rp\rp^2
  =\int_{0}^{\infty}\int_{0}^{a_i^{(1)}} \lp a_i^{(1)}-  a_i^{(2)} \rp^2 f_{\chi}(a_i^{(2)}) f_{\chi}(a_i^{(1)})da_i^{(2)}da_i^{(1)}.\label{eq:quad13}
\end{eqnarray}

 We summarize the above results in the following lemma.
\begin{lemma}(Memory capacity upper bound; \textbf{quadratic} activation) Assume the setup of Theorem \ref{thm:thm1}. For quadratic $\f^{(2)}(\x)=\x^2$, let $c(d;\mbox{quad})\triangleq c(d;\f^{(2)}(\x)=\x^2)$ be the $n$-scaled memory capacity from (\ref{eq:model4}). Let $a_i^{(1)}$ and $a_i^{(2)}$ be independent, chi distributed, random variables with $\frac{d}{2}$ degrees of freedom and let $f_{\chi}(a)$ be as in (\ref{eq:quad12}). One then has the following
\vspace{-.0in}
\vspace{-.0in}\begin{center}
\tcbset{beamer,lower separated=false, fonttitle=\bfseries,
coltext=black , interior style={top color=orange!10!yellow!30!white, bottom color=yellow!80!yellow!50!white}, title style={left color=orange!10!cyan!30!blue, right color=green!70!blue!20!black}}
 \begin{tcolorbox}[beamer,title=\textbf{($n$-scaled) memory capacity upper bound:},lower separated=false, fonttitle=\bfseries,width=.92\linewidth] 
\vspace{-.15in}
 \begin{eqnarray*}
\hspace{-.0in} \hat{c}(d;\mbox{quad})=  \frac{2}{\mE \lp \max\lp a_i^{(1)}-  a_i^{(2)},0\rp\rp^2}
=  \frac{2}{\int_{0}^{\infty}\int_{0}^{a_i^{(1)}} \lp a_i^{(1)}-  a_i^{(2)} \rp^2 f_{\chi}(a_i^{(2)}) f_{\chi}(a_i^{(1)})da_i^{(2)}da_i^{(1)}}. \end{eqnarray*}
\vspace{-.15in}
 \end{tcolorbox}
\end{center}\vspace{-.0in}
Then for any sample complexity $m$ such that $\alpha\triangleq \lim_{n\rightarrow\infty}\frac{m}{n}>\hat{c}(d;\mbox{quad})$
\begin{eqnarray}
 \lim_{n\rightarrow\infty}\mP_{X}(A([n,d,1];\mbox{quad}) \quad \mbox{fails to memorize data set} \quad (X,\1))\longrightarrow 1,\label{eq:lemma3quadta30}
\end{eqnarray}
and
\begin{eqnarray}
 \lim_{n\rightarrow\infty}\mP_{X}(c(d,\mbox{quad})<\hat{c}(d,\mbox{quad}))\longrightarrow 1.\label{eq:lemma3quadta30a}
\end{eqnarray}
 \label{lemma:lemma3}
\end{lemma}\vspace{-.17in}
\begin{proof}
Follows immediately from the above discussion.
\end{proof}
Taking, say, $d=2$ for the concreteness, one finds $\hat{c}(d;\mbox{quad})=5.4978 $, which basically means that when the sample complexity $m$ is such that $m>5.4978n$ (with $n$ being the data vectors' ambient dimension) the network fails to memorize the data. Consequently, one has for the memory capacity of the $2$ hidden layer \emph{quadratically} activated neurons TCMs, $C(A([n,2,1];\mbox{quad}))\leq 5.4978n$. The results for a wider range of $d$ are shown in Figure \ref{fig:fig2}. We also add the $\hat{c}(\infty;\mbox{quad})=4$, replica symmetry based prediction obtained in \cite{ZavPeh21}. As the figure indicates, one has that the RDT upper bound from the above theorem approaches the $d\rightarrow\infty$ replica symmetry based prediction already for fairly narrow nets with the number of neurons of the order of a couple of tens. We should also add that due to the fact that the underlying problems are now highly non-convex, the strong random duality considerations from \cite{StojnicRegRndDlt10,StojnicUpper10,StojnicGorEx10} are inapplicable. As we will see a bit later on, the results obtained above (and shown in Figure \ref{fig:fig2}), are in fact \emph{strict} (non-tight) capacity upper bounds.

\begin{figure}[h]
\centering
\centerline{\includegraphics[width=1\linewidth]{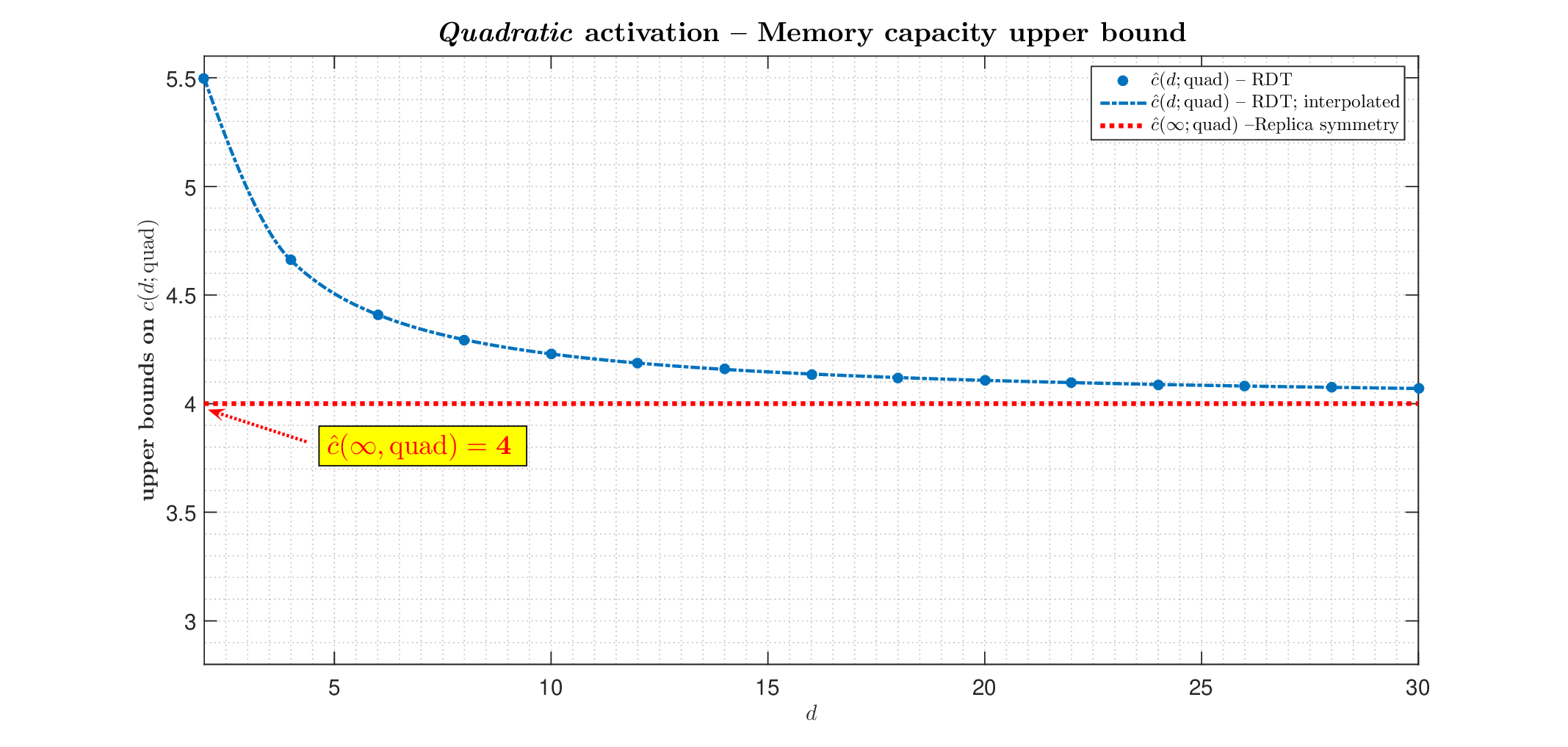}}
\caption{Memory capacity upper bound as a function of the number of neurons, $d$, in the hidden layer; 1-hidden layer TCM with \textbf{\emph{quadratic}} activations; \bl{\textbf{\emph{plain}  RDT}} estimate (\red{\textbf{Replica symmetry (RS)}} $d\rightarrow\infty$ estimate is included as well)}
\label{fig:fig2}
\end{figure}

\subsubsection{ReLU hidden layer activations -- $\f^{(2)}(\x)=\max(\x,0)$}
\label{sec:relu}

As was the case above when we considered the quadratic activation, for ReLU (rectified linear unit) one has  $\f^{(2)}(\x)=\max(\x,0)\geq 0$. This means that one again needs to carefully make a choice for vector $\w$ which ensures that the network functioning makes sense. Moreover, one again easily observes that some of the components of $\w$  must be negative. We follow the trend set above and in the ReLU relevant literature and consider $\w$ with $\frac{d}{2}$ 1s and $\frac{d}{2}$ -1s, i.e., we again take
 \begin{equation}\label{eq:relu0}
\w=\begin{bmatrix}
     -\1 \\ \1
   \end{bmatrix},
\end{equation}
while assuming that the size of the column vectors $\1$ is $\frac{d}{2}\times 1$. After again relying on (\ref{eq:ta19}), (\ref{eq:ta21}), (\ref{eq:lin1}), and (\ref{eq:lin2}), one recognizes the following key optimization problem of interest
 \begin{equation}\label{eq:relu3}
  z_i(\g;\mbox{relu})  =  z_i(G_{i,1:d})  =\min_{\f^{(2)}(Q_{i,1:d}) \w\geq 0}  \sum_{j=1}^{d}(G_{ij}-Q_{ij})^2
  =\min_{\f^{(2)}(\q^T) \w\geq 0}  \|\g-\q\|_2^2  =\min_{\max(\q,0)^T\w\geq 0}  \|\g-\q\|_2^2.
\end{equation}
Splitting the problem into two parts, gives the following formulation
 \begin{eqnarray}\label{eq:relu3a1}
  z_i(\g;\mbox{relu})   =  \min_{\q} & &   \sum_{i=1}^{\frac{d}{2}}(\g_i-\q_i)^2+ \sum_{i=\frac{d}{2}+1}^d(\g_i-\q_i)^2 \nonumber \\
\mbox{subject to}   & &  \sum_{i=1}^{\frac{d}{2}}  \max(\q_i,0) \leq  \sum_{i=\frac{d}{2}+1}^d \max(\q_i,0).
\end{eqnarray}

\subsubsubsection{$d=2$}
\label{sec:d2}

We study separately the simplest case $d=2$. There are two reasons for doing so: 1) One can obtain a neat and explicit closed form final result; and 2) Somewhat counter-intuitively, it will turn out that the bounding capacities are decreasing functions of $d$.

When $d=2$, (\ref{eq:relu3a1}) becomes
 \begin{eqnarray}\label{eq:relud2a1}
  z_i(\g;\mbox{relu}_2)   =  \min_{\q} & &   (\g_1-\q_1)^2+ (\g_2-\q_2)^2 \nonumber \\
\mbox{subject to}   & & \max(\q_1,0) \leq  \max(\q_2,0).
\end{eqnarray}
For $\g_1\leq 0$ one easily has $ z_i(\g)=0$. Also, one trivially has that for $\g_2\geq \g_1$, $ z_i(\g)=0$. We then focus on scenario where $\g_1\geq 0$ and $\g_2\leq \g_1$ happen simultaneously. One first finds that for $\g_2\leq (1-\sqrt{2}) \g_1$, $ z_i(\g)=\g_1^2$. On the other hand for $ (1-\sqrt{2}) \g_1 \geq \g_2\leq \g_1$, $ z_i(\g)=\min\lp\g_1^2,2\lp\frac{\g_1-\g_2}{2}\rp^2\rp=\frac{\lp\g_1-\g_2\rp^2}{2}$. In a more compact form one then has
 \begin{eqnarray}\label{eq:relud2a1}
  z_i(\g;\mbox{relu}_2)   =  \begin{cases}
                 0, & \mbox{if } \g_1\leq 0 \quad \mbox{or} \quad \g_2\geq \g_1\geq 0 \\
                 \g_1^2, & \mbox{if } \g_1\geq 0 \quad \mbox{and} \quad \g_2\leq (1-\sqrt{2}) \g_1 \\
                 \frac{\lp\g_1-\g_2\rp^2}{2}, & \mbox{if } \g_1\geq 0 \quad \mbox{and} \quad (1-\sqrt{2}) \g_1\leq \g_2\leq \g_1.
               \end{cases}.
\end{eqnarray}
Connecting (\ref{eq:ta18a0}), (\ref{eq:ta20}), (\ref{eq:relu3}), (\ref{eq:relu3a1}), and (\ref{eq:relud2a1}), we obtain
\begin{eqnarray}
\phi_0 & = &   \lim_{n\rightarrow \infty}\mE_{G} \frac{1}{\sqrt{n}}
 \min_{\phi(Q)=0} \|G-Q\|_F- 1 \nonumber \\
 & = &   \lim_{n\rightarrow \infty}\mE_{G} \frac{1}{\sqrt{n}}
 \sqrt{\sum_{i=1}^{m}  z_i(G_{i,1:d})}- 1 \nonumber \\
 & = &    \sqrt{\alpha \mE   z_i(\g;\mbox{relu}_2)  }- 1.\label{eq:relud1a11}
\end{eqnarray}
Moreover, we also have
\begin{eqnarray}
\mE   z_i(\g;\mbox{relu}_2)
& = &   \int_{0}^{\infty}\int_{-\infty}^{\infty}  z_i(\g;\mbox{relu}_2) \frac{e^{-\frac{\g_2^2}{2}}}{\sqrt{2\pi}}\frac{e^{-\frac{\g_1^2}{2}}}{\sqrt{2\pi}}d\g_2d\g_1 \nonumber \\
& = &   \int_{0}^{\infty}
\lp \int_{-\infty}^{(1-\sqrt{2})\g_1}\g_1^2  \frac{e^{-\frac{\g_2^2}{2}}}{\sqrt{2\pi}}
+ \int_{(1-\sqrt{2})\g_1}^{\g_1} \frac{\lp\g_1-\g_2\rp^2}{2}  \frac{e^{-\frac{\g_2^2}{2}}}{\sqrt{2\pi}}
\rp
\frac{e^{-\frac{\g_1^2}{2}}}{\sqrt{2\pi}}d\g_2d\g_1 \nonumber \\
& = &   \int_{0}^{\infty}
\lp I_1(\g_1)+I_2(\g_1)
\rp
\frac{e^{-\frac{\g_1^2}{2}}}{\sqrt{2\pi}}d\g_1,\label{eq:relud1a13}
\end{eqnarray}
where
\begin{eqnarray}
I_1(\g_1) & \triangleq & \g_1^2 \frac{\erfc\lp \frac{(\sqrt{2}-1)\g_1}{\sqrt{2}} \rp}{2} \nonumber \\
I_2(\g_1) & \triangleq & \frac{1}{4}\lp(\g_1^2 + 1)\lp\erf\lp\frac{\g_1}{\sqrt{2}}\rp-\erf\lp\frac{(1-\sqrt{2})\g_1}{\sqrt{2}}\rp\rp + \sqrt{\frac{2}{\pi}}\lp \g_1 e^{-\frac{\g_1^2}{2}} - (1+\sqrt{2})\g_1e^{-\frac{\lp (1-\sqrt{2})\g_1 \rp^2}{2}} \rp \rp. \nonumber \\
\label{eq:relud1a14}
\end{eqnarray}

 We summarize the above results in the following lemma.
\begin{lemma}(Memory capacity upper bound; \textbf{ReLU} activation; $d=2$) Assume the setup of Theorem \ref{thm:thm1}. For rectified linear unit (ReLU) $\f^{(2)}(\x)=\max(\x,0)$, let $c(d;\mbox{relu})\triangleq c(d;\f^{(2)}(\x)=\max(\x,0)))$ be the $n$-scaled memory capacity from (\ref{eq:model4}). Let $I_1(\g_1)$ and $I_2(\g_1)$ be as in (\ref{eq:relud1a14}). One then has the following
\vspace{-.0in}
\vspace{-.0in}\begin{center}
\tcbset{beamer,lower separated=false, fonttitle=\bfseries,
coltext=black , interior style={top color=orange!10!yellow!30!white, bottom color=yellow!80!yellow!50!white}, title style={left color=orange!10!cyan!30!blue, right color=green!70!blue!20!black}}
 \begin{tcolorbox}[beamer,title=\textbf{($n$-scaled) memory capacity upper bound:},lower separated=false, fonttitle=\bfseries,width=.65\linewidth] 
\vspace{-.15in}
 \begin{eqnarray*}
\hspace{-.0in} \hat{c}(2;\mbox{relu})
=  \frac{1}{\int_{0}^{\infty}
\lp I_1(\g_1)+I_2(\g_1)
\rp
\frac{e^{-\frac{\g_1^2}{2}}}{\sqrt{2\pi}}d\g_1}\approx \mathbf{3.81}. \end{eqnarray*}
\vspace{-.15in}
 \end{tcolorbox}
\end{center}\vspace{-.0in}
Then for any sample complexity $m$ such that $\alpha\triangleq \lim_{n\rightarrow\infty}\frac{m}{n}>\hat{c}(2;\mbox{relu})$
\begin{eqnarray}
 \lim_{n\rightarrow\infty}\mP_{X}(A([n,2,1];\mbox{relu}) \quad \mbox{fails to memorize data set} \quad (X,\1))\longrightarrow 1,\label{eq:lemma4reluta30}
\end{eqnarray}
and
\begin{eqnarray}
 \lim_{n\rightarrow\infty}\mP_{X}(c(2,\mbox{relu})<\hat{c}(2,\mbox{relu}))\longrightarrow 1.\label{eq:lemma4reluta30a}
\end{eqnarray}
 \label{lemma:lemma4}
\end{lemma}\vspace{-.17in}
\begin{proof}
Follows immediately from the above discussion.
\end{proof}
The above lemma basically states that when the sample complexity $m$ is such that $m>3.81n$ (with $n$ being the data vectors' ambient dimension) the network fails to memorize the data. Consequently, one has for the memory capacity of the $2$ hidden layer \emph{ReLU} activated neurons TCMs, $C(A([n,2,1];\mbox{relu}))\leq 3.81n$. As was the case when we discussed the quadratic activation, due to highly non-convex underlying problems, the strong random duality considerations from \cite{StojnicRegRndDlt10,StojnicUpper10,StojnicGorEx10} are inapplicable.

\subsubsubsection{General $d$}
\label{sec:gend}

For studying general $d$, we find it convenient to introduce vector $\g^{(1)}$, $\g^{(2)}$, $\g^{(2,ac)}$, and $\g^{(2,a)}$
 \begin{eqnarray}\label{eq:relugendeq1}
\g^{(1)} & \triangleq & \left\{ \g_i|\g_i>0,1\leq i\leq \frac{d}{2}\right\} \nonumber \\
\g^{(2)} & \triangleq & \mbox{sort}\lp \g_{\frac{d}{2}+1:d} \rp \nonumber \\
\g^{(2,a)} & \triangleq & \min(\g^{(2)},0),
\end{eqnarray}
where sorting is in the descending order. Basically, $\g^{(1)}$ is comprised of the positive components of $\g_{1:\frac{d}{2}}$ and $\g^{(2)}$ is $\g_{\frac{d}{2}+1:d}$ sorted in the descending order. Also, for the notational simplicity, let the lengths of $\g^{(1)}$ and $\g^{(2)}$ be $d_1$  and $d_2$, respectively and let $d_3$ be the number of the nonnegative elements of $\g^{(2)}$ (clearly, $d_2=\frac{d}{2}$). It is then not difficult to see that (\ref{eq:relu3a1}) can be rewritten as
 \begin{eqnarray}\label{eq:relugendeq1a1}
  z_i(\g;\mbox{relu})   =  \min_{\q^{(1)},\q^{(2)}} & &   \sum_{i=1}^{d_1}(\g^{(1)}_i-\q^{(1)}_i)^2+ \sum_{i=1}^{d_2}(\g^{(2,a)}_i-\q^{(2)}_i)^2 \nonumber \\
\mbox{subject to}   & &  \sum_{i=1}^{d_1}  \max(\q^{(1)}_i,0) \leq  \sum_{i=1}^{d_2} \max(\q^{(2)}_i,0).
\end{eqnarray}
Moreover, given the positivity of $\g^{(1)}$, it is relatively easy to see that
 \begin{eqnarray}\label{eq:relugendeq2}
  z_i(\g;\mbox{relu})   =  \min_{\q^{(1)}\geq 0,\q^{(2)}} & &   \sum_{i=1}^{d_1}(\g^{(1)}_i-\q^{(1)}_i)^2+ \sum_{i=1}^{d_2}(\g^{(2,a)}_i-\q^{(2)}_i)^2 \nonumber \\
\mbox{subject to}   & &  \sum_{i=1}^{d_1}  \q^{(1)}_i \leq  \sum_{i=1}^{d_2} \max(\q^{(2)}_i,0).
\end{eqnarray}
Following the methodology utilized for studying the quadratic activations, we find it convenient to study the following optimization problem for a $b\geq 0$ and for any $k\in\{d_3,d_3+1,\dots,d_2\}$
  \begin{eqnarray}\label{eq:relu3a2}
  z_i^{(2)}(\g^{(2)},b,k;\mbox{relu})   =  \min_{\q^{(2)}\geq 0} & &   \sum_{i=1}^k(\g^{(2,a)}_i-\q^{(2)}_i)^2 \nonumber \\
\mbox{subject to}   & &  \sum_{i=1}^k \q^{(2)}_i=b.
\end{eqnarray}
One can then trivially rewrite (\ref{eq:relu3a2}) as
 \begin{eqnarray}\label{eq:relu3a3}
  z_i^{(2)}(\g^{(2)},b,k;\mbox{relu})   =  \min_{\q^{(2)}\geq 0} & &   \sum_{i=1}^k \lp \g^{(2,a)}_i\rp^2-2\sum_{i=1}^k\g^{(2,a)}_i\q^{(2)}_i +\sum_{i=1}^k \lp\q^{(2)}_i\rp^2 \nonumber \\
\mbox{subject to}   & &  \sum_{i=1}^k \q^{(2)}_i=b.
\end{eqnarray}
After writing the Lagrangian and relying on the strong duality, we also have
 \begin{eqnarray}\label{eq:relu3a3a1}
  z_i^{(2)}(\g^{(2)},b,k;\mbox{relu})
  & =  & \min_{\q^{(2)}\geq 0} \max_{\nu}   \sum_{i=1}^k\lp\g^{(2,a)}_i\rp^2-2\sum_{i=1}^k\g^{(2,a)}_i\q^{(2)}_i +\sum_{i=1}^k\lp \q^{(2)}_i\rp^2 +2\nu \lp\sum_{i=1}^k \q^{(2)}_i-b\rp \nonumber \\
  & =  & \max_{\nu}  \min_{\q^{(2)}\geq 0}  \sum_{i=1}^k\lp \g^{(2,a)}_i\rp^2-2\sum_{i=1}^k\g^{(2,a)}_i\q^{(2)}_i +\sum_{i=1}^k\lp\q^{(2)}_i\rp^2 +2\nu \lp\sum_{i=1}^k \q^{(2)}_i-b\rp.
\end{eqnarray}
Solving over $\q^{(2)}$ gives
 \begin{eqnarray}\label{eq:relu3a3a2}
\q^{(2,opt)}=\max(\g^{(2,a)}-\nu,0),
\end{eqnarray}
and
 \begin{eqnarray}\label{eq:relu3a3a3}
  z_i^{(2)}(\g^{(2)},b,k;\mbox{relu})
  & =  & \max_{\nu} \lp  \sum_{i=1}^k\lp\g^{(2,a)}_i\rp^2-\sum_{i=1}^k\lp\q^{(2,opt)}_i\rp^2 -2\nu b\rp.
\end{eqnarray}
After setting
 \begin{eqnarray}\label{eq:relu3a3a4}
  \bar{z}_i^{(2)}(\g^{(2)},b,k;\mbox{relu})
   =   \begin{cases}
  \max_{\nu}  \lp \sum_{i=1}^k\lp\g^{(2,a)}_i\rp^2-\sum_{i=1}^k\lp\q^{(2,opt)}_i\rp^2 -2\nu b\rp, & \mbox{if } \min(\q^{(2,opt)})>0 \\
            \infty, & \mbox{otherwise},
         \end{cases}
\end{eqnarray}
it is not that difficult to see that (\ref{eq:relugendeq2}) can be rewritten as
 \begin{eqnarray}\label{eq:relugendeq3}
  z_i(\g;\mbox{relu})   =  \min_{b\geq 0,\q^{(1)}\geq 0} & &   \sum_{i=1}^{d_1}(\g^{(1)}_i-\q^{(1)}_i)^2+ \min_{k\in\{d_3,d_3+1,\dots,d_2\}} z_i^{(2)}(\g^{(2)},b,k;\mbox{relu})\nonumber \\
\mbox{subject to}   & &  \sum_{i=1}^{d_1}  \q^{(1)}_i \leq  b.
\end{eqnarray}
We can then write the Lagrangian and rely on the strong duality as above to obtain
 \begin{eqnarray}\label{eq:relugendeq4}
  z_i(\g;\mbox{relu})
  &  =  & \min_{b\geq 0, \q^{(1)}\geq 0} \max_{\nu_1\geq 0}   \sum_{i=1}^{d_1}(\g^{(1)}_i-\q^{(1)}_i)^2+ \min_{k\in\{d_3+1,d_3+2,\dots,d_2\}} \bar{z}_i^{(2)}(\g^{(2)},b,k;\mbox{relu}) +2\nu_1(\sum_{i=1}^{d_1}  \q^{(1)}_i - b) \nonumber \\
 &  =  & \min_{b\geq 0}\max_{\nu_1\geq 0} \min_{\q^{(1)}\geq 0}    \sum_{i=1}^{d_1}(\g^{(1)}_i-\q^{(1)}_i)^2+ \min_{k\in\{d_3+1,d_3+2,\dots,d_2\}} \bar{z}_i^{(2)}(\g^{(2)},b,k;\mbox{relu}) +2\nu_1(\sum_{i=1}^{d_1}  \q^{(1)}_i - b).\nonumber \\
\end{eqnarray}
Solving over $\q^{(1)}$ gives
 \begin{eqnarray}\label{eq:relugendeq5}
\q^{(1,opt)}=\max(\g^{(1)}-\nu_1,0),
\end{eqnarray}
and
 \begin{eqnarray}\label{eq:relugendeq6}
  z_i(\g;\mbox{relu})
  &  =  & \min_{b\geq 0}\max_{\nu_1\geq 0}    \sum_{i=1}^{d_1}\lp\g^{(1)}_i\rp^2- \sum_{i=1}^{d_1}\lp\q^{(1,opt)}_i\rp^2 -2\nu_1 b + \min_{k\in\{d_3,d_3+1,\dots,d_2\}} \bar{z}_i^{(2)}(\g^{(2)},b,k;\mbox{relu}).\nonumber \\
\end{eqnarray}
Connecting (\ref{eq:ta18a0}), (\ref{eq:ta20}), (\ref{eq:relu3}), (\ref{eq:relu3a1}), and (\ref{eq:relugendeq6}), we obtain
\begin{eqnarray}
\phi_0 & = &   \lim_{n\rightarrow \infty}\mE_{G} \frac{1}{\sqrt{n}}
 \min_{\phi(Q)=0} \|G-Q\|_F- 1 \nonumber \\
 & = &   \lim_{n\rightarrow \infty}\mE_{G} \frac{1}{\sqrt{n}}
 \sqrt{\sum_{i=1}^{m}  z_i(G_{i,1:d})}- 1 \nonumber \\
 & = &    \sqrt{\alpha \mE   z_i(\g;\mbox{relu})  }- 1.\label{eq:relugenda11}
\end{eqnarray}

 We summarize the above results in the following lemma.
\begin{lemma}(Memory capacity upper bound; \textbf{ReLU} activation; general $d$) Assume the setup of Theorem \ref{thm:thm1}. For rectified linear unit (ReLU) $\f^{(2)}(\x)=\max(\x,0)$, let $c(d;\mbox{relu})\triangleq c(d;\f^{(2)}(\x)=\max(\x,0)))$ be the $n$-scaled memory capacity from (\ref{eq:model4}). Let $\g$ be an $d$-dimensional vector comprised of iid standard normals and let $\g^{(1)}$, $\g^{(2)}$, and $\g^{(2,a)}$ be as in (\ref{eq:relugendeq1}). Also, let $d_3$ be the number of the nonnegative elements of $\g^{(2)}$ and let $d_2=\frac{d}{2}$. Additionally, let $\q^{(2,opt)}$, $\q^{(1,opt)}$, $\bar{z}_i^{(2)}(\g^{(2)},b,k;\mbox{relu})$, and $z_i(\g;\mbox{relu})$  be as in (\ref{eq:relu3a3a2}), (\ref{eq:relugendeq5}), (\ref{eq:relu3a3a4}), and (\ref{eq:relugendeq6}), respectively. One then has the following
\vspace{-.0in}
\vspace{-.0in}\begin{center}
\tcbset{beamer,lower separated=false, fonttitle=\bfseries,
coltext=black , interior style={top color=orange!10!yellow!30!white, bottom color=yellow!80!yellow!50!white}, title style={left color=orange!10!cyan!30!blue, right color=green!70!blue!20!black}}
 \begin{tcolorbox}[beamer,title=\textbf{($n$-scaled) memory capacity upper bound:},lower separated=false, fonttitle=\bfseries,width=.65\linewidth] 
\vspace{-.15in}
 \begin{eqnarray*}
\hspace{-.0in} \hat{c}(d;\mbox{relu})
=  \frac{1}{\mE   z_i(\g;\mbox{relu}) }. \end{eqnarray*}
\vspace{-.15in}
 \end{tcolorbox}
\end{center}\vspace{-.0in}
Then for any sample complexity $m$ such that $\alpha\triangleq \lim_{n\rightarrow\infty}\frac{m}{n}>\hat{c}(d;\mbox{relu})$
\begin{eqnarray}
 \lim_{n\rightarrow\infty}\mP_{X}(A([n,d,1];\mbox{relu}) \quad \mbox{fails to memorize data set} \quad (X,\1))\longrightarrow 1,\label{eq:lemma5reluta30}
\end{eqnarray}
and
\begin{eqnarray}
 \lim_{n\rightarrow\infty}\mP_{X}(c(d,\mbox{relu})<\hat{c}(d,\mbox{relu}))\longrightarrow 1.\label{eq:lemma5reluta30a}
\end{eqnarray}
 \label{lemma:lemma5}
\end{lemma}\vspace{-.17in}
\begin{proof}
Follows immediately from the above discussion.
\end{proof}
Utilization of the above lemma relies on a solid amount of numerical work. Taking for the concreteness, say, $d=4$, we obtain $\hat{c}(4;\mbox{relu})\approx 3.11$. This basically means that when the sample complexity $m$ is such that $m>3.11n$ (with $n$ being the data vectors' ambient dimension) the network fails to memorize the data. Consequently, one has for the memory capacity of the $4$ hidden layer \emph{ReLU} activated neurons TCMs, $C(A([n,4,1];\mbox{relu}))\leq 3.11n$. One can continue for other even $d$, with the numerical calculations being more and more involved as $d$ increases. A bit easier (albeit not as precise and a bit jittery) alternative is to simulate higher values of $d$. The obtained results are shown for a wider range of $d$ in Figure \ref{fig:fig3}. The replica symmetry based prediction, $\hat{c}(\infty;\mbox{relu})\approx 2.93$, obtained in \cite{ZavPeh21} is shown for the completeness as well. Finally, we should add that (as was the case when we discussed the quadratic activation and the \emph{ReLU} one with $d=2$), due to highly non-convex underlying problems, the strong random duality considerations from \cite{StojnicRegRndDlt10,StojnicUpper10,StojnicGorEx10} are inapplicable.

\begin{figure}[h]
\centering
\centerline{\includegraphics[width=1\linewidth]{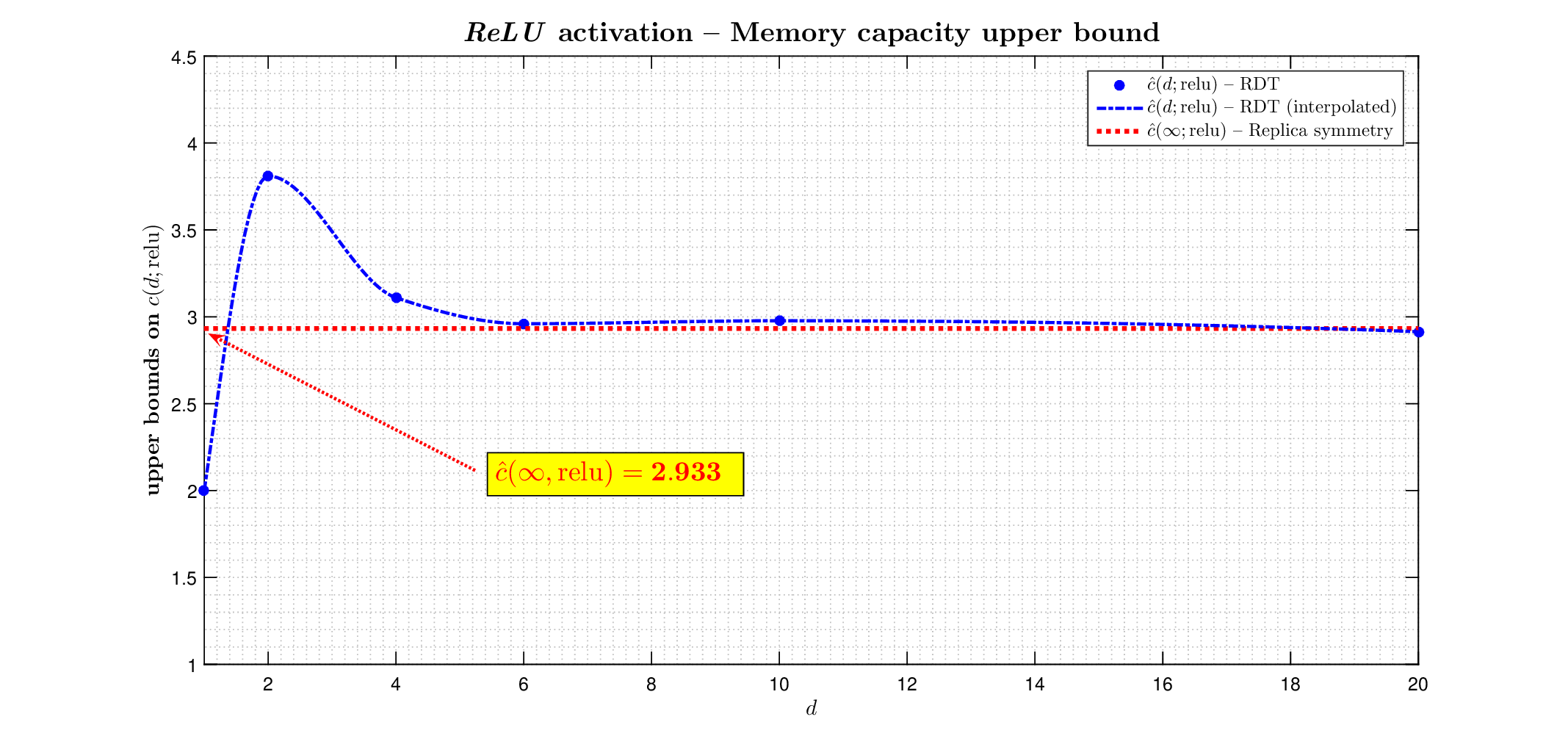}}
\caption{Memory capacity upper bound as a function of the number of neurons, $d$, in the hidden layer; 1-hidden layer TCM with \textbf{\emph{ReLU}} activations; \bl{\textbf{\emph{plain}  RDT}} estimate (\red{\textbf{Replica symmetry (RS)}} $d\rightarrow\infty$ estimate is included as well)}
\label{fig:fig3}
\end{figure}

\section{\emph{Partially lifted} Random Duality Theory (pl RDT)}
\label{sec:liftedrdt}

As the results from the previous sections showed, the RDT is rather useful tool when it comes to characterizing the memory capacity. In particular, the \emph{plain} RDT determines the memory capacity for the \emph{linear} activation and upper-bounds it for the \emph{quadratic} and \emph{ReLU} activations. Moreover, the scenarios where the application of the \emph{plain} RDT is of the upper-bounding nature can be handled through the recently developed \emph{fully lifted} (fl) RDT \cite{Stojnicsflgscompyx23,Stojnicnflgscompyx23,Stojnicflrdt23}. However, one needs to keep in mind, that the fl RDT relies on heavy numerical evaluations which would come on top of the already seen substantial numerical work from the previous sections. Opting for an analytically less accurate but computationally more convenient route seems as practically more beneficial. Recalling that a similar situation was observed when the \emph{sign} activations were considered  \cite{Stojnictcmspnncaprdt23,Stojnictcmspnncapliftedrdt23}, we find it reasonable to follow the path taken overthere and consider a \emph{partially lifted} (pl) RDT variant which relies on the following principles.
\vspace{-.0in}\begin{center}
 	\tcbset{beamer,lower separated=false, fonttitle=\bfseries, coltext=black ,
		interior style={top color=yellow!20!white, bottom color=yellow!60!white},title style={left color=black!80!purple!60!cyan, right color=yellow!80!white},
		width=(\linewidth-4pt)/4,before=,after=\hfill,fonttitle=\bfseries}
 \begin{tcolorbox}[beamer,title={\small Summary of the \emph{partially lifted} (pl) RDT's main principles} \cite{StojnicCSetam09,StojnicRegRndDlt10,StojnicLiftStrSec13,StojnicGardSphErr13,StojnicGardSphNeg13}, width=1\linewidth]
\vspace{-.15in}
{\small \begin{eqnarray*}
 \begin{array}{ll}
\hspace{-.19in} \mbox{1) \emph{Finding underlying optimization algebraic representation}}
 & \hspace{-.0in} \mbox{2) \emph{Determining the \textbf{partially lifted} random dual}} \\
\hspace{-.19in} \mbox{3) \emph{Handling the \textbf{partialy lifted} random dual}} &
 \hspace{-.0in} \mbox{4) \emph{Double-checking strong random duality.}}
 \end{array}
  \end{eqnarray*}}
\vspace{-.25in}
 \end{tcolorbox}
\end{center}\vspace{-.0in}

We below assume a solid level of familiarity with the discussions presented in \cite{Stojnictcmspnncaprdt23,Stojnictcmspnncapliftedrdt23}. To ensure the smoothness of the presentation, we parallel the presentation from  \cite{Stojnictcmspnncapliftedrdt23} as closely as possible and discuss separately each of the above four principles within the context of our interest here.

\vspace{.1in}
\noindent \underline{1) \textbf{\emph{Algebraic memorization characterization:}}}  This part of the pl RDT corresponds to the first part of the plain RDT and is already obtained in Lemma \ref{lemma:lemma1}. As mentioned earlier, Lemma \ref{lemma:lemma1} holds for any given data set $\left (\x^{(0,k)},1\right )_{k=1:m}$. On the other hand, to analyze (\ref{eq:ta10}) and (\ref{eq:ta11}), the pl RDT proceeds similarly to the plain RDT and imposes a statistics  on $X$.


\noindent \underline{2) \textbf{\emph{Determining the partially lifted random dual:}}} Keeping in mind the measure concentration from (\ref{eq:ta15})  (see, e.g. \cite{StojnicCSetam09,StojnicRegRndDlt10,StojnicICASSP10var,Stojnictcmspnncaprdt23,Stojnictcmspnncapliftedrdt23}), the following so-called \emph{partially lifted} random dual theorem is another key ingredient of the RDT machinery.
\begin{theorem}(Memorization characterization via \emph{partially lifted} random dual) Let $d$ be any even integer. Consider a TCM with $d$ neurons in the hidden layer and architecture $A([n,d,1];\f^{(2)})$, and let the elements of $X\in\mR^{m\times n}$, $G\in\mR^{m\times d}$, and $H\in\mR^{\delta\times d}$ be iid standard normals. Assuming $c_3>0$ and $\w\in\mR^{d\times 1}$, set
\vspace{-.0in}
\begin{eqnarray}
\phi(Q) & \triangleq & \|\1-\emph{\mbox{sign}}(\f^{(2)}(Q) \w)\|_2\nonumber \\
 f_{rd}^{(1)}(G) & \triangleq &  \max_{\phi(Q)=0}  -c_3  \|G-Q\|_F  \nonumber \\
 f_{rd}^{(2)}(H) & \triangleq &  \|H\|_F  \nonumber \\
 \bar{\phi}_0(\alpha;c_3)  & \triangleq & \lim_{n\rightarrow\infty} \frac{1}{\sqrt{n}}\lp  \frac{c_3}{2}
- \frac{1}{c_3}\log \lp \mE_{G}e^{\frac{c_3}{2}f_{rd}^{(1)}(G)}\rp
- \frac{1}{c_3}\log \lp \mE_{H}e^{\frac{c_3}{2}f_{rd}^{(2)}(H)}\rp   \rp.\label{eq:ta16}
\vspace{-.04in}\end{eqnarray}
One then has \vspace{-.02in}
\begin{eqnarray}
\hspace{-.1in}(\bar{\phi}_0(\alpha;c_3)  > 0)   &  \Longrightarrow  &  \lp \lim_{n\rightarrow\infty}\mP_{X}(f_{rp}>0)\longrightarrow 1 \rp  \nonumber \\
& \Longrightarrow & \lp \lim_{n\rightarrow\infty}\mP_{X}(A([n,d,1];\f^{(2)}) \quad \mbox{fails to memorize data set} \quad (X,\1))\longrightarrow 1\rp.\label{eq:ta17}
\end{eqnarray}
\label{thm:thm2}
\end{theorem}\vspace{-.17in}
\begin{proof}
    Immediate consequence of Theorem 2 from \cite{Stojnictcmspnncapliftedrdt23}.
      \end{proof}

%
%
%
%
%
\noindent \underline{3) \textbf{\emph{Handling the lifted random dual:}}} After proceeding with a detailed careful analysis of the optimization over $Q$, one arrives at the following theorem.
\begin{theorem}(Memory capacity partially lifted (pl) RDT based upper bound; general $d$) Assume the setup of Theorem \ref{thm:thm2}. Let the network $n$-scaled capacity, $c(d;\f^{(2)})$, be as defined in (\ref{eq:model4}) and let $\g$ be a $d$-dimensional vector of iid standard normals. First one has
 \begin{eqnarray}\label{eq:thm3aan12}
  z_i(\g;\f^{(2)}) & = & \min_{\f^{(2)}(\q^T) \w\geq 0}  \|\g-\q\|_2^2 \nonumber \\
 I_Q & = &  \mE_{\g} e^{-\frac{c_3}{4\gamma} z_i\lp \g;\f^{(2)}\rp }  \nonumber \\
 I_{sph} & = & \gamma_{sph}-\frac{1}{2c_3}\log \lp 1-\frac{c_3}{2\gamma_{sph}}\rp, \quad  \gamma_{sph} =  \frac{c_3+\sqrt{c_3^2+4}}{4} \nonumber \\
\bar{\phi}_0(\alpha) &  = & \max_{c_3>0}\min_{\gamma} \lp \frac{c_3}{2} +\gamma -\frac{\alpha}{c_3}\log(I_Q)-I_{sph}\rp.
\end{eqnarray}
 Further, consider the following
\vspace{-.0in}
\vspace{-.0in}\begin{center}
\tcbset{beamer,lower separated=false, fonttitle=\bfseries,
coltext=black , interior style={top color=orange!10!yellow!30!white, bottom color=yellow!80!yellow!50!white}, title style={left color=orange!10!cyan!30!blue, right color=green!70!blue!20!black}}
 \begin{tcolorbox}[beamer,title=\textbf{($n$-scaled general $d$) memory capacity upper bound, $\bar{c}(d;\f^{(2)})$, that satisfies:},lower separated=false, fonttitle=\bfseries,width=.92\linewidth] 
\vspace{-.15in}
{\small\begin{eqnarray}\label{eq:thm3aan13}
\bar{\phi}_0(\bar{c}(d;\f^{(2)}))=0 \quad \Longleftrightarrow \quad  \max_{c_3>0}\min_{\gamma} \lp \frac{c_3}{2} +\gamma -\frac{\bar{c}(d;\f^{(2)})}{c_3}\log(I_Q)-I_{sph}\rp = 0. \end{eqnarray}}
\vspace{-.15in}
 \end{tcolorbox}
\end{center}\vspace{-.0in}
Then for any sample complexity $m$ such that $\alpha\triangleq \lim_{n\rightarrow\infty}\frac{m}{n}>\bar{c}(d;\f^{(2)})$
\begin{eqnarray}
 \lim_{n\rightarrow\infty}\mP_{X}(A([n,d,1];\f^{(2)}) \quad \mbox{fails to memorize data set} \quad (X,\1))\longrightarrow 1,\label{eq:thm3ta36}
\end{eqnarray}
and
\begin{eqnarray}
 \lim_{n\rightarrow\infty}\mP_{X}(c(d;\f^{(2)})<\bar{c}(d;\f^{(2)}))\longrightarrow 1.\label{eq:thm3ta37}
\end{eqnarray}
\label{thm:thm3}
\end{theorem}\vspace{-.0in}

\begin{proof}
The proof is split into two parts: \textbf{\emph{(i)}} Handling $\frac{1}{c_3\sqrt{n}}\log\lp\mE_{H} \mbox{exp}\lp  c_3 f_{rd}^{(2)}(H)\rp \rp$; and \textbf{\emph{(ii)}} Handling of $\frac{1}{c_3\sqrt{n}}\log\lp\mE_{G} \mbox{exp}\lp c_3 f_{rd}^{(1)}(G)\rp\rp$.

\underline{\textbf{\textbf{\emph{(i)}} Handling $\frac{1}{c_3\sqrt{n}}\log\lp\mE_{H} \mbox{exp}\lp  c_3 f_{rd}^{(2)}(H)\rp \rp$:}} One first observes
\begin{equation}\label{eq:supp8}
I_{sph} \triangleq \frac{1}{c_3\sqrt{n}}\log\lp\mE_{H} \mbox{exp}\lp  c_3 f_{rd}^{(2)}(H)\rp \rp =
 \frac{1}{c_3\sqrt{n}}\log\lp\mE_{H} \mbox{exp}\lp  c_3 \|H^T\|_F\rp \rp.
\end{equation}
After appropriate scaling, $c_3\rightarrow c_3\sqrt{n}$,  it was determined in \cite{StojnicMoreSophHopBnds10,Stojnictcmspnncapliftedrdt23} that
 \begin{equation}\label{eq:supp9}
  I_{sph}  =  \gamma_{sph}-\frac{1}{2c_3}\log \lp 1-\frac{c_3}{2\gamma_{sph}}\rp, \quad  \gamma_{sph} =  \frac{c_3+\sqrt{c_3^2+4}}{4}.
\end{equation}

\underline{\textbf{\textbf{\emph{(ii)}} Handling $\frac{1}{c_3\sqrt{n}}\log\lp\mE_{G} \mbox{exp}\lp  c_3 f_{rd}^{(1)}(G)\rp \rp$:}} Following closely \cite{Stojnictcmspnncapliftedrdt23}, we first observe
\begin{equation}\label{eq:supp10}
\log(I_{Q}') \triangleq \frac{1}{c_3\sqrt{n}} \log \lp \mE_{G} \mbox{exp}\lp  c_3 f_{rd}^{(2)}(G)\rp\rp =
 \frac{1}{c_3\sqrt{n}} \log \lp\mE_{G} \mbox{exp}\lp  -c_3 \min_{\phi(Q)=0}\|G-Q\|_F\rp \rp.
\end{equation}
Utilizing the squaring trick introduced on many occasions in   \cite{StojnicGardSphErr13,StojnicMoreSophHopBnds10}, we further find
\begin{equation}\label{eq:supp11}
\log(I_{Q}')=\max_{\gamma}
 \frac{1}{c_3\sqrt{n}}\log \lp \mE_{G} \mbox{exp}\lp  c_3\lp -\frac{1}{4\gamma} \min_{\phi(Q)=0}\|G-Q\|_F^2 -\gamma  \rp\rp\rp.
\end{equation}
Keeping in mind the appropriate scaling, $c_3\rightarrow c_3\sqrt{n}$ and $\gamma\rightarrow \gamma\sqrt{n}$, and recalling $\alpha=\lim_{n\rightarrow\infty}\frac{m}{n}$, one also has
\begin{equation}\label{eq:supp12}
-\log(I_{Q}')=\min_{\gamma} \lp\gamma -
 \frac{\alpha}{c_3}\log \lp \mE_{G} \mbox{exp}\lp  c_3\lp -\frac{1}{4\gamma} \min_{\phi_i(Q_{i,:})=1}\|G_{i,:}-Q_{i,:}\|_F^2 \rp\rp\rp\rp,
\end{equation}
where
\begin{equation}\label{eq:supp13}
  \phi_i(Q_{i,:})\triangleq \mbox{sign}(\mbox{sign}(Q_{i,:})\1).
\end{equation}
It is then not difficult to see that (\ref{eq:supp12}) is equivalent to the following
\begin{eqnarray}\label{eq:supp14}
-\log(I_{Q}') & = & \min_{\gamma} \lp\gamma -
 \frac{\alpha}{c_3}\log \lp \mE_{\g} \mbox{exp}\lp  \lp -\frac{c_3}{4\gamma} \min_{\f^{(2)}(\q^T)\w\leq 0}\|\g-\q\|_2^2 \rp\rp\rp\rp \nonumber \\
& = & \min_{\gamma} \lp\gamma -
 \frac{\alpha}{c_3}\log \lp \mE_{\g} \mbox{exp}\lp  -\frac{c_3}{4\gamma} z_i \lp \g;\f^{(2)} \rp \rp\rp\rp \nonumber \\
 & = & \min_{\gamma} \lp\gamma -
 \frac{\alpha}{c_3}\log \lp I_Q\rp \rp.
 \end{eqnarray}
A simple combination of (\ref{eq:ta16}), (\ref{eq:supp8})-(\ref{eq:supp10}), and (\ref{eq:supp14}) then completes the proof.
 \end{proof}


\noindent \underline{4) \textbf{\emph{Double checking the strong random duality:}}}  As discussed earlier and in \cite{Stojnictcmspnncaprdt23,Stojnictcmspnncapliftedrdt23}, the standard strong random duality double checking is not in place due to inapplicability of the typical, convexity based, considerations from \cite{StojnicRegRndDlt10,StojnicUpper10,StojnicGorEx10}.

\subsection{Specialization to particular $\f^{(2)}$ activations}
\label{sec:diffactplrdt}

Theorem \ref{thm:thm3} is generic and works for various $\f^{(2)}$ activations. To see how the whole machinery practically works, we here consider particular $\f^{(2)}$ activations. However, since the plain RDT completely solved the linear activation, we here focus only on the remaining two, the \emph{quadratic} and the \emph{ReLU}. In fact, we first focus most of our interest to the quadratic one as in that case the concrete capacity results can be obtained without an extensive numerical work. We then afterwards briefly comment on the ReLU case as well.

\subsubsection{Pl RDT capacity estimates for quadratic activations -- $\f^{(2)}(\x)=\x^2$}
\label{sec:quadplrdt}

The following theorem summarizes the pl RDT results for the \emph{quadratic} activations.

\begin{theorem}(Memory capacity partially lifted (pl) RDT based upper bound; \textbf{quadratic} activation) Assume the setup of Lemma \ref{lemma:lemma3} and Theorem \ref{thm:thm3} with $a_i^{(1)}$ and $a_i^{(2)}$ being independent chi distributed random variables with $\frac{d}{2}$ degrees of freedom. First one has
 \begin{eqnarray}\label{eq:thm4aan12}
 I_Q & = &  \int_{0}^{\infty}\int_{0}^{\infty} e^{-\frac{c_3}{4\gamma}\frac{\lp \max( a_i^{(1)}-  a_i^{(2)},0) \rp^2}{2}} f_{\chi}(a_i^{(2)}) f_{\chi}(a_i^{(1)})da_i^{(2)}da_i^{(1)}  \nonumber \\
 I_{sph} & = & \gamma_{sph}-\frac{1}{2c_3}\log \lp 1-\frac{c_3}{2\gamma_{sph}}\rp, \quad  \gamma_{sph} =  \frac{c_3+\sqrt{c_3^2+4}}{4} \nonumber \\
\bar{\phi}_0(\alpha) &  = & \max_{c_3>0}\min_{\gamma} \lp \frac{c_3}{2} +\gamma -\frac{\alpha}{c_3}\log(I_Q)-I_{sph}\rp.
\end{eqnarray}
 Further, consider the following
\vspace{-.0in}
\vspace{-.0in}\begin{center}
\tcbset{beamer,lower separated=false, fonttitle=\bfseries,
coltext=black , interior style={top color=orange!10!yellow!30!white, bottom color=yellow!80!yellow!50!white}, title style={left color=orange!10!cyan!30!blue, right color=green!70!blue!20!black}}
 \begin{tcolorbox}[beamer,title=\textbf{($n$-scaled general $d$) memory capacity upper bound, $\bar{c}(d;\mbox{quad})$, that satisfies:},lower separated=false, fonttitle=\bfseries,width=.92\linewidth] 
\vspace{-.15in}
\begin{eqnarray}\label{eq:thm4aan13}
\bar{\phi}_0(\bar{c}(d;\mbox{quad}))=0 \quad \Longleftrightarrow \quad  \max_{c_3>0}\min_{\gamma} \lp \frac{c_3}{2} +\gamma -\frac{\bar{c}(d;\mbox{quad})}{c_3}\log(I_Q)-I_{sph}\rp = 0. \end{eqnarray}
\vspace{-.15in}
 \end{tcolorbox}
\end{center}\vspace{-.0in}
Then for any sample complexity $m$ such that $\alpha\triangleq \lim_{n\rightarrow\infty}\frac{m}{n}>\bar{c}(d;\mbox{quad})$
\begin{eqnarray}
 \lim_{n\rightarrow\infty}\mP_{X}(A([n,d,1];\mbox{quad}) \quad \mbox{fails to memorize data set} \quad (X,\1))\longrightarrow 1,\label{eq:thm4ta36}
\end{eqnarray}
and
\begin{eqnarray}
 \lim_{n\rightarrow\infty}\mP_{X}(c(d;\mbox{quad})<\bar{c}(d;\mbox{quad}))\longrightarrow 1.\label{eq:thm4ta37}
\end{eqnarray}
\label{thm:thm4}
\end{theorem}\vspace{-.17in}

\begin{proof}
  Follows immediately from Theorem \ref{thm:thm3} after recognizing that
  $z_i(\g;\mbox{quad}) = \frac{\lp \max\lp a_i^{(1)}-  a_i^{(2)},0\rp\rp^2}{2}$.
\end{proof}

The results obtained based on the above theorem for a wider range of $d$ are shown in Figure \ref{fig:fig4}. For the completeness and a quick comparison, we include the results obtained earlier based on the plain RDT. The benefit of the partially lifted RDT is fairly strong throughout the entire range of the considered $d$'s. We also add the partial 1rsb $d\rightarrow\infty$ estimate obtained based on the statistical physics replica methods in \cite{ZavPeh21}. As can be seen from the figure, the convergence with $d$ is rather fast and already for fairly small $d$ values (of the order of a couple of tens) the limiting, $d\rightarrow\infty$, bound is narrowly approached. As was the case for the plain RDT, we here observe that the bounding capacity estimates are decreasing as $d$ increases with the largest value obtained again for $d=2$. Given its clear importance, we, for the concreteness, take precisely $d=2$ and find $\bar{c}(d;\mbox{quad})=4.065$, which then implies that the memory capacity of the $2$ hidden layer \emph{quadratically} activated neurons TCMs, $C(A([n,2,1];\mbox{quad}))\leq 4.065n$. Moreover, one observes a substantial drop from the plain RDT bound of $5.4978$ established in earlier sections.

\begin{figure}[h]
\centering
\centerline{\includegraphics[width=1\linewidth]{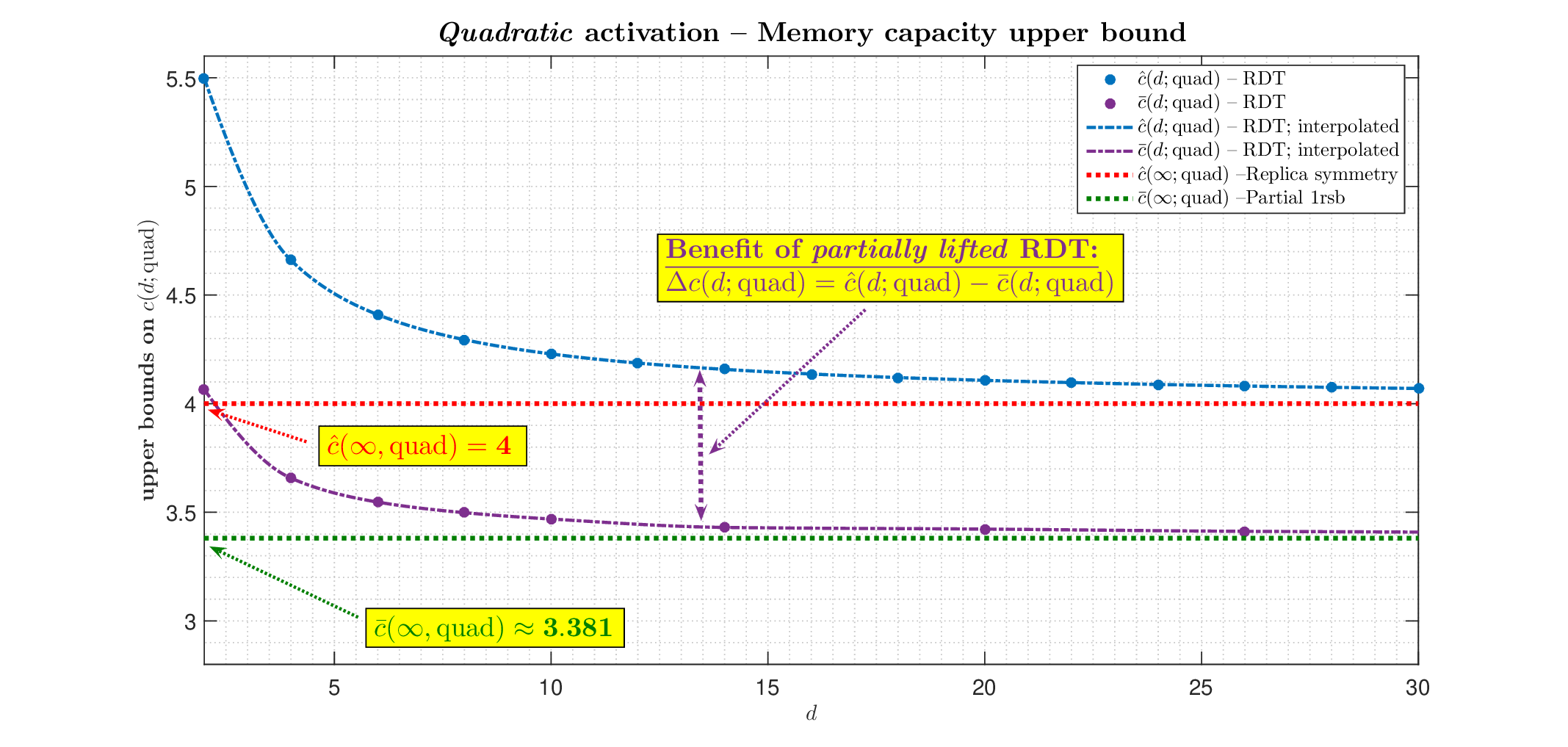}}
\caption{Memory capacity upper bound as a function of the number of neurons, $d$, in the hidden layer; 1-hidden layer TCM with \textbf{\emph{quadratic}} activations; \bl{\textbf{\emph{plain}  RDT}} versus \prp{\textbf{\emph{partially lifted} RDT}} (\red{\textbf{Replica symmetry (RS)}} and \dgr{\textbf{Partial 1rsb}} $d\rightarrow\infty$ estimates are included as well)}
\label{fig:fig4}
\end{figure}

\subsubsection{Pl RDT capacity estimates for ReLU activations -- $\f^{(2)}(\x)=\max(\x,0)$}
\label{sec:reluplrdt}

The following theorem summarizes the pl RDT results for the \emph{ReLU} activations.

\begin{theorem}(Memory capacity partially lifted (pl) RDT based upper bound; \textbf{ReLU} activation) Assume the setup of Lemma \ref{lemma:lemma3} and Theorem \ref{thm:thm3}  with $z_i(\g;\mbox{relu})$ as in (\ref{eq:relugendeq6}). First one has
 \begin{eqnarray}\label{eq:thm5aan12}
 I_Q & = &  \int_{-\infty}^{\infty} e^{-\frac{c_3}{4\gamma} z_i(\g;\mbox{relu}) }d\g  \nonumber \\
 I_{sph} & = & \gamma_{sph}-\frac{1}{2c_3}\log \lp 1-\frac{c_3}{2\gamma_{sph}}\rp, \quad  \gamma_{sph} =  \frac{c_3+\sqrt{c_3^2+4}}{4} \nonumber \\
\bar{\phi}_0(\alpha) &  = & \max_{c_3>0}\min_{\gamma} \lp \frac{c_3}{2} +\gamma -\frac{\alpha}{c_3}\log(I_Q)-I_{sph}\rp.
\end{eqnarray}
 Further, consider the following
\vspace{-.0in}
\vspace{-.0in}\begin{center}
\tcbset{beamer,lower separated=false, fonttitle=\bfseries,
coltext=black , interior style={top color=orange!10!yellow!30!white, bottom color=yellow!80!yellow!50!white}, title style={left color=orange!10!cyan!30!blue, right color=green!70!blue!20!black}}
 \begin{tcolorbox}[beamer,title=\textbf{($n$-scaled general $d$) memory capacity upper bound, $\bar{c}(d;\mbox{relu})$, that satisfies:},lower separated=false, fonttitle=\bfseries,width=.92\linewidth] 
\vspace{-.15in}
\begin{eqnarray}\label{eq:thm5aan13}
\bar{\phi}_0(\bar{c}(d;\mbox{relu}))=0 \quad \Longleftrightarrow \quad  \max_{c_3>0}\min_{\gamma} \lp \frac{c_3}{2} +\gamma -\frac{\bar{c}(d;\mbox{relu})}{c_3}\log(I_Q)-I_{sph}\rp = 0. \end{eqnarray}
\vspace{-.15in}
 \end{tcolorbox}
\end{center}\vspace{-.0in}
Then for any sample complexity $m$ such that $\alpha\triangleq \lim_{n\rightarrow\infty}\frac{m}{n}>\bar{c}(d;\mbox{relu})$
\begin{eqnarray}
 \lim_{n\rightarrow\infty}\mP_{X}(A([n,d,1];\mbox{relu}) \quad \mbox{fails to memorize data set} \quad (X,\1))\longrightarrow 1,\label{eq:thm5ta36}
\end{eqnarray}
and
\begin{eqnarray}
 \lim_{n\rightarrow\infty}\mP_{X}(c(d,\mbox{relu})<\bar{c}(d,\mbox{relu}))\longrightarrow 1.\label{eq:thm5ta37}
\end{eqnarray}
\label{thm:thm5}
\end{theorem}\vspace{-.17in}

\begin{proof}
  Follows immediately from Theorem \ref{thm:thm3} after recognizing that
  $z_i(\g;\mbox{relu}) $ from (\ref{eq:relugendeq6}) is precisely $z_i(\g;\f^{(2)}=\max(\x,0))$.
\end{proof}

The numerical evaluations are now substantially more involved even for small values of $d$. Moreover, the simulations are rather extensive for larger values and the indication from the plain RDT suggest that $d=2$ is particularly relevant. As indicated in Table \ref{tab:tab1}, we applied the above machinery for $d=2$ and obtained the bound $\approx 3.81$, which means that in this particular case the partial RDT makes no improvement over the plain RDT. The table is completed by taking the plain RDT value for $d=4$ as well, since the underlying partial RDT numerical work is already rather heavy.

\section{Conclusion}
\label{sec:conc}

In this paper we studied the treelike committee machines (TCM) neural networks  and their memory capabilities. Differently form the common practice, we here instead of typical \emph{sign} perceptron hidden layer neuronal activations consider a generic set of activations. Utilizing a powerful mathematical concept called Random Duality Theory (RDT), \cite{Stojnictcmspnncaprdt23} established a generic statistical framework for the 1-hidden layer TCMs analysis that can study on a very precise level the scaled capacities for any given number of the neurons in the hidden layer, $d$. Among other things, such a machinery effectively enabled avoiding the qualitative/descriptive scaling types of estimates typically prevalent in the capacity analysis literature. Moreover, studying the \emph{sign} perceptron activations, it also made a very strong progress towards obtaining, in a mathematically rigorous way, their \emph{exact} $n$-scaled capacities. For small values of $d$, it also made a very first rigorous progress in over 30 years over the previously best known bounds of \cite{MitchDurb89}. Since the results of \cite{Stojnictcmspnncaprdt23} are, in general, of the upper-bounding type, \cite{Stojnictcmspnncapliftedrdt23} proceeded further by considering the so-called partially lifted (pl) RDT variant and significantly lowered the estimates from  \cite{Stojnictcmspnncaprdt23}. Such a lowering further resulted in ensuring a universal (over the entire range of $d$) improvement over the previously best known results of \cite{MitchDurb89}.

We here adopt the same strategy and utilize both the plain RDT and the partial RDT to characterize the 1-hidden layer TCM capacities with neuronal activations substantially different from the classical \emph{sign} one. We first establish a universal framework for studying generic activations and then consider three particular activations types that have attracted a strong interest in recent NN literature: (i) \emph{linear}; (ii) \emph{quadratic}; and (iii) \emph{ReLU}. For the linear activation we show that the plain RDT \emph{exactly} characterizes the capacity. Moreover, we show that, no matter how wide the hidden layer is, the capacity remains equal to the capacity of the single spherical \emph{sign} perceptron. For the quadratic and ReLU activations we obtain that the plain RDT predictions are decreasing functions of $d$ that converge to a constant value. The maximum bounding value is in both cases obtained for the smallest possible $d=2$. Moreover, for the pl RDT and quadratic activation, we obtain a strong improvement over the plain RDT through the entire range of the considered $d$'s. At the same time, the bounding capacity maintains the decreasing in $d$ property with the maximal value again being achieved for (the minimal possible) $d=2$. For the ReLU, we obtained that the pl RDT offers no improvement over the plain RDT for $d=2$ which means that the same, decreasing in $d$,  trend applies to these activations as well. Moreover, we uncover that another of the trends observed in  \cite{Stojnictcmspnncaprdt23} manifests itself here as well. Namely, the bounds obtained in \cite{Stojnictcmspnncaprdt23} precisely matched the corresponding statistical physics replica symmetry based predictions obtain in \cite{EKTVZ92,BHS92}. Here, we also observe  that the \emph{linear} activation predictions precisely match the ones obtained through the replica considerations in \cite{ZavPeh21,BalMalZech19}. Moreover, the $d\rightarrow\infty$ converging values of our both plain RDT and pl RDT closely approach the corresponding ones obtained in \cite{ZavPeh21}.

 Various extensions are possible as well. It is rather clear that the first next one is to conduct the analysis with the \emph{fully lifted} (fl) RDT (see, e.g., \cite{Stojnicflrdt23}). Also, we here consider only three well known activation functions. Many others are of interest as well, e.g., sigmoid, erf, tanh and so on. More complex multi-layered network architectures including both TCM and  FCM or PM based ones are of interest as well. All of these extensions, we will discuss in separate papers.

\begin{singlespace}
\bibliographystyle{plain}
\bibliography{nflgscompyxRefs}
\end{singlespace}

\end{document}